\documentclass[10pt,twocolumn,letterpaper]{article}

\PassOptionsToPackage{table}{xcolor}

\usepackage[pagenumbers]{cvpr} 
\usepackage{fontawesome5}
\usepackage{tabularx}
\usepackage{array}
\usepackage[table]{xcolor}
\usepackage[T1]{fontenc}
\usepackage[utf8]{inputenc}
\usepackage{microtype}
\usepackage{mathtools}
\usepackage{amsthm}
\newtheorem{proposition}{Proposition}
\usepackage{multirow}
\usepackage{tabularx}
\usepackage{algorithm}
\usepackage{algpseudocode}

\definecolor{ovcblue}{HTML}{285F8F}
\newcommand{\ovcdelta}[1]{\textcolor{ovcblue}{\ensuremath{#1}}}

\definecolor{cvprblue}{rgb}{0.21,0.49,0.74}
\usepackage[pagebackref,breaklinks,colorlinks,allcolors=cvprblue]{hyperref}

\def\confName{CVPR}
\def\confYear{2027}

\title{From Visual Affinity to Learning Credit in Multimodal Reinforcement Learning}

\author{
Yuhang Han\textsuperscript{1,2,3,*},
Dingning Liu\textsuperscript{5,*},
Yujie Chen\textsuperscript{1,2},
Yaojie Zhang\textsuperscript{5},
Jingwei Song\textsuperscript{2}\\
Zequn Yang\textsuperscript{2,3},
Xuanming Shang\textsuperscript{2,3},
Li Ding\textsuperscript{2,3},
Zhijiang Guo\textsuperscript{1,4,\ensuremath{\dagger}},
Linfeng Zhang\textsuperscript{2,\ensuremath{\dagger}}\\[5pt]
\textsuperscript{1}The Hong Kong University of
Science and Technology (Guangzhou)\\
\textsuperscript{2}Shanghai Jiao Tong University\\
\textsuperscript{3}Shadoweave\\
\textsuperscript{4}The Hong Kong University of
Science and Technology\\
\textsuperscript{5}University of Electronic
Science and Technology of China\\[4pt]
{\small
\textsuperscript{*}Equal contribution.
\quad
\textsuperscript{\ensuremath{\dagger}}Corresponding authors.
}\\[6pt]
{\small
\faGlobe\enspace
\url{https://kawhiiiileo.github.io/OVC_PAGE/}
}\\[2pt]
{\small
\faGithub\enspace
\url{https://github.com/kawhiiiileo/OVC}
}
}

\begin{document}
\maketitle
\begin{abstract}
Reinforcement learning with verifiable rewards (RLVR) improves
multimodal reasoning, but uniformly assigning a single response-level
advantage to all tokens overlooks the distinct roles of visual
observations and subsequent reasoning.
Visual affinity alone is insufficient: predictive states may align
with both image and prompt-context representations, and upweighting
valid visual observations can amplify their negative-advantage
updates when later reasoning fails.
To address this issue, we propose \textbf{Outcome-Conditioned Visual Credit (OVC)},
a plug-in framework that separates visual-association measurement
from token-level credit assignment and is adaptable to
policy-optimization methods with uniform response-level advantage
allocation.
Relative Visual Encoding constructs a context-relative visual
coordinate by contrasting each predictive state's affinity to
image and prompt-context representations.
Outcome-conditioned Evidence Calibration combines this coordinate
with uncertainty-scaled predictive features and calibrates their
contributions to outcome-discriminative evidence through mean
contrasts between positive- and negative-advantage responses,
scaled by regularized within-side variation.
Visual association thus informs credit through observed outcome
differences rather than directly determining update priority.
Normalized Credit Allocation conditions the calibrated evidence
on the response's advantage sign to produce positive token weights
with unit global mean.
These detached weights preserve advantage signs while modulating
individual token-gradient contributions.
OVC reuses existing rollouts without auxiliary reward models,
process annotations, or inference-time modules.
Instantiated through GSPO-token, OVC retains sequence-level clipping
and consistently improves six-benchmark mean accuracy over GSPO
across three models, with gains of 1.86--2.39 percentage points.
\end{abstract}

    
\section{Introduction}
\label{sec:introduction}

Reinforcement learning with verifiable rewards (RLVR)~\citep{DAPO,GRPO,GSPO,AT-RL} 
has become a key approach to improving reasoning in large
vision language models (LVLMs)~\citep{qwen3,qwen35,innovator-vl}. 
It provides verifiable outcome supervision 
without dense intermediate-step annotations.
Recent efforts have advanced multimodal reasoning 
via improved training data~\citep{data1,data2,data3,data4,data5,Mm-eureka,vision-r1}, 
policy optimization~\citep{reward1,reward2,reward3,reward4,reward5,reward6,reward7}, 
and perceptual supervision~\citep{PAPO,VisualPRM}. 
Yet outcome feedback remains coarser than the computations it guides. 
A single response may involve visual observations,
symbolic transformations, and intermediate inferences, 
while its reward evaluates only the complete solution. 
This motivates finer-grained token credit assignment using visual signals.

Existing approaches~\citep{VPPO,AT-RL,KAWHI,VGPO} derive
token-level signals from visual dependence, representational similarity,
predictive uncertainty, or region alignment.
Although these signals capture different aspects of visual involvement,
their conversion into update coefficients still requires deciding how
the measured association should affect learning.
With DelTA's predictive-state representation~\citep{DelTA},
a state may align with both image and prompt representations, so
absolute image affinity alone does not distinguish these associations.
Moreover, a valid visual observation may occur in a
negative-advantage response when later reasoning fails.
Upweighting it would strengthen its negative update even if the error
occurs elsewhere in the response.
The central issue is therefore
\textbf{not merely how to measure visual association,
but how to interpret it under outcome feedback.}
\begin{figure}[t]
  \centering
  \includegraphics[
      width=\linewidth,
      trim=0 15pt 0 9pt,
      clip
    ]{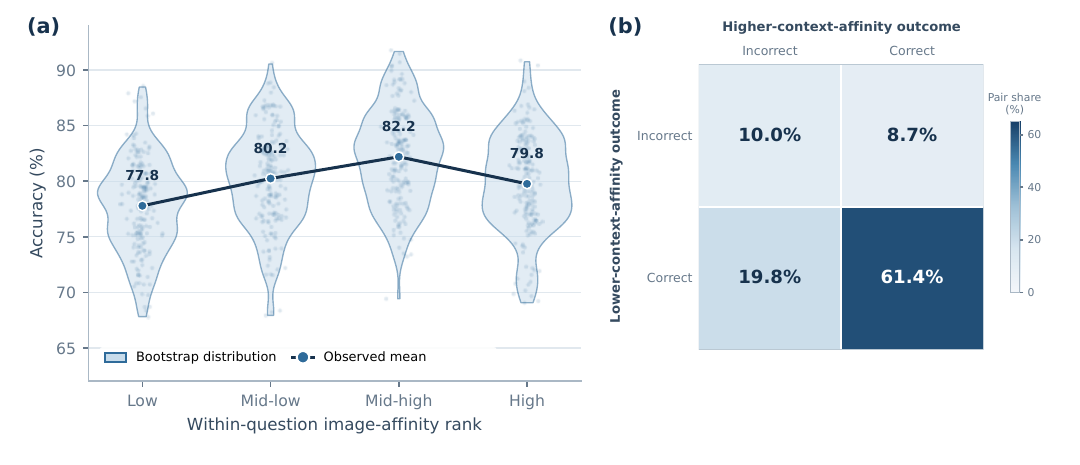}
\caption{
\textbf{Affinity and response correctness.}
We analyze 8 valid responses per randomly sampled question
from a frozen Qwen3.5-9B trained with GSPO on Math77K,
without correctness filtering.
Affinity is
$s_i^u=|\mathcal R_i|^{-1}\sum_{t\in\mathcal R_i}
\cos(h_{it},g_i^u)$,
where $h_{it}$ predicts token $t$, $g_i^u$ is the mean-pooled
image ($u=\mathrm{img}$) or non-image prompt ($u=\mathrm{ctx}$)
prototype from the same intermediate layer, and $\mathcal R_i$
indexes reasoning tokens excluding special tokens and the final answer.
No coordinate rescaling or image--context subtraction is applied.
\textbf{(a)} Mean accuracy over four within-question
image-affinity rank groups (two responses each).
Each faint point is a group mean from a question-level bootstrap
resample, with shared indices across groups.
Violins show the distributions; dark markers are
original-sample means.
\textbf{(b)} Joint correctness of within-question pairs
roughly matched on image affinity.
Rows and columns represent lower- and higher-context-affinity
response outcomes, respectively; cells report percentages
of all matched pairs. 
}
  \label{fig:motivation}
\end{figure}

Figure~\ref{fig:motivation} illustrates this distinction.
In (a), the highest image-affinity group does not attain the
highest observed mean accuracy, suggesting a possible mismatch
between affinity ranking and response quality.
In (b), responses approximately matched on image affinity
exhibit different outcome statistics across context-affinity levels.
These observations motivate considering contextual association
alongside image affinity, rather than treating the latter alone
as sufficient for credit allocation.
Together with the advantage-sign dependence discussed above,
they motivate a two-stage principle: measuring visual relevance
relative to context and calibrating it against outcome differences.

To instantiate this principle, we propose
\textbf{Outcome-Conditioned Visual Credit (OVC)}.
Its first component, \textbf{Relative Visual Encoding (RVE)},
contrasts a predictive state's cosine affinity to image and
prompt-background representations under a shared diagonal
scaling. The resulting coordinate captures relative visual
association without directly prescribing an update magnitude.
\textbf{Outcome-conditioned Evidence Calibration (OEC)} then
combines this coordinate with predictive features and determines
their evidential contributions from contrasts between positive-
and negative-advantage responses. A regularized estimate of
within-side diagonal variation calibrates the relative
contribution of each coordinate. 
Consequently, with the statistics held fixed, 
the visual contribution depends on the outcome contrast 
and the variance scale, not on affinity alone, 
and equals zero when the visual reference means coincide.

This calibrated evidence is converted into policy updates
through \textbf{Normalized Credit Allocation (NCA)}.
NCA produces positive, globally normalized token coefficients
that modulate the original sequence advantage while preserving
its sign. Applied through GSPO-token~\citep{GSPO}, these detached coefficients
retain sequence-level importance ratios and clipping while
adjusting the contribution of individual token-gradient terms.
OVC introduces no auxiliary reward model, additional process
or region supervision, or inference-time module.
To clarify this reweighting, we characterize the resulting gradient
and separate response-level weighting from within-response allocation,
enabling controlled evaluation of token--coefficient correspondence.

Our contributions are threefold.
\begin{itemize}
    \item We introduce a context-relative visual coordinate
    that contrasts image affinity with prompt-background
    affinity, providing an explicit visual measure for
    credit estimation.

    \item We develop an outcome-conditioned calibration
    method combining visual and predictive features
    via advantage-side contrasts and regularized
    within-side variation, with no fixed positive update preference for visual affinity.

    \item We characterize how calibrated evidence enters
    policy updates and distinguish response-level reweighting
    from allocation within a response, establishing an
    analytical basis for controlled evaluation of token
    credit assignment.
\end{itemize}
\section{Related Work}
\paragraph{Policy Optimization and Perceptual Supervision for Multimodal Reasoning.}
RLVR improves reasoning through outcome verification. GRPO~\citep{GRPO} employs group-relative advantage estimation, while GSPO~\citep{GSPO} introduces sequence-level importance ratios for policy optimization. In multimodal settings, MM-Eureka~\citep{Mm-eureka} investigates large-scale training with rule-based rewards, and VL-Rethinker~\citep{Vl-rethinker} improves training through selective sample replay and forced reflection. Another line of work refines supervision. PAPO~\citep{PAPO} constructs an implicit perceptual objective through image perturbations, whereas VisualPRM~\citep{VisualPRM} trains a separate process reward model to evaluate intermediate steps and support candidate solution selection. These studies primarily advance optimization mechanisms, training procedures, or supervision sources. Complementing these efforts, our work investigates how visual representations can inform fine-grained credit estimation under outcome feedback, rather than redesigning sampling strategies or introducing an independent process evaluator.

\paragraph{Visually Guided Fine-Grained Credit Assignment.}
VPPO~\citep{VPPO} and PGPO~\citep{PGPO} exploit changes in predictive distributions across visual conditions to filter token updates or redistribute advantages. PRPO~\citep{PRPO}{} further combines strong and weak perturbations to distinguish visual dependence from prediction instability. VGPO~\citep{VGPO} derives attention-related scores from visual prototypes and combines late-stage compensation with advantage reweighting within and across trajectories. \texttt{KAWHI}~\citep{KAWHI} connects structured visual evidence to RLVR through geometry-guided key-region selection, query-key alignment in vision-critical heads, and paragraph-level credit redistribution. These methods explicitly connect visual signals to policy updates. However, visual sensitivity, representational affinity, and region alignment do not, by themselves, determine a token's contribution to outcome improvement. Our approach instead treats visual association as a feature requiring calibration. We construct a context-relative visual coordinate and determine its evidential contribution through statistical contrasts between positive- and negative-advantage responses and within-side variance, thereby separating the measurement of visual relevance from the allocation of learning credit.
\section{Method}
\label{sec:ovc}

\begin{figure*}[t]
    \centering
    \includegraphics[width=\linewidth]{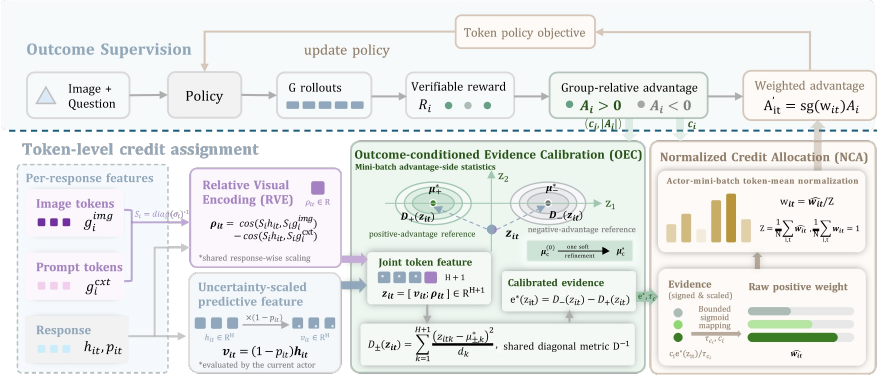}
    \caption{\textbf{Overview of Outcome-Conditioned Visual Credit (OVC).} An image--question pair produces a group of responses whose rewards yield group-relative advantages. For each response token, OVC combines relative image--context affinity with a prediction-aware representation, calibrates this evidence across advantage groups, and assigns normalized token weights to the policy update.}
    \label{fig:pipeline}
\end{figure*}

Visual relevance describes an association with the image, whereas
learning credit determines the strength of a policy update. We treat
the former as a measurement to be calibrated against outcome feedback,
rather than as a quantity to maximize.
We first characterize the calibration rule,
then specify its estimator and policy objective.

\paragraph{Setting and notation.}
For each multimodal input $x=(I,q)$, a behavior policy
$\pi_{\mathrm{old}}$ samples $G$ responses.
Response $y_i=(y_{it})_{t=1}^{T_i}$ receives a verifiable reward $R_i$
and a group-relative advantage~\citep{GRPO}
\begin{equation}
    A_i
    =
    \frac{R_i-\overline R_{g(i)}}{s_{g(i)}+\epsilon_A},
    \qquad
    c_i=\operatorname{sign}(A_i),
    \label{eq:advantage}
\end{equation}
where $g(i)$ contains responses to the same input, $\overline R_{g(i)}$
and $s_{g(i)}$ are the corresponding reward mean and standard deviation,
and $\epsilon_A>0$.
The sign $c_i$ indicates relative reward, not token correctness.
An actor mini-batch contains $B$ nonempty responses and
$N=\sum_i T_i$ valid tokens.
Advantages are computed within each rollout group, whereas the
reference statistics and coefficient normalization are aggregated
over the actor mini-batch.
All credit statistics are detached and fixed during its optimizer
step. Sums below exclude padding.

\subsection{A Calibration Principle for Visual Credit}
\label{sec:principle}

Let $z_{it}=[v_{it};\rho_{it}]$ combine a predictive feature
$v_{it}$ with a scalar visual measurement $\rho_{it}$.
For $c\in\{-1,+1\}$, let
$\mathcal T_c=\{(i,t)\mid c_i=c,\ 1\leq t\leq T_i\}$
contain the valid response tokens on the corresponding
advantage side.
We initialize the reference means using advantage-magnitude
weights and assume $M_c>0$ for both advantage sides
\begin{equation}
\begin{aligned}
    \mu_c^{(0)}
    &= \frac{1}{M_c}
    \sum_{(i,t)\in\mathcal T_c}|A_i|z_{it},\\
    M_c
    &= \sum_{(i,t)\in\mathcal T_c}|A_i|.
\end{aligned}
\label{eq:reference-initialization}
\end{equation}
These statistics are aggregated across the actor mini-batch, rather
than estimated separately for each input.
The reference-based calibration below assumes $M_+>0$ and $M_->0$.
The two sides are defined by the sign of the sequence advantage,
not by token-level correctness labels.

We write $\mu_+,\mu_-$ generically for the reference means
before or after refinement.
Their refinement and the construction of a positive diagonal
matrix $D$ encoding regularized within-side variation are
specified in Section~\ref{sec:estimator}.
We calibrate a token feature against both advantage sides by
comparing its squared distances to the two references under
the shared metric $D^{-1}$:
\begin{equation}
    e(z)
    =
    \|z-\mu_-\|_{D^{-1}}^2
    -
    \|z-\mu_+\|_{D^{-1}}^2,
    \label{eq:calibration}
\end{equation}
where $\|a\|_{D^{-1}}^2=a^\top D^{-1}a$.
Positive evidence indicates that $z$ is closer to the
positive-advantage reference, while negative evidence indicates
the opposite.

For fixed reference statistics, define the variance-calibrated direction
\[
    b^*=D^{-1}\Delta\mu,
    \qquad
    \Delta\mu=\mu_+-\mu_-.
\]
Direct expansion of the distance contrast in
Eq.~\eqref{eq:calibration} gives
\begin{equation}
    e(z)
    =
    2(b^*)^\top(z-\overline\mu),
    \qquad
    \overline\mu=\frac{\mu_++\mu_-}{2}.
    \label{eq:evidence}
\end{equation}
Each coordinate therefore contributes through its advantage-side
mean contrast relative to its regularized within-side scale, rather
than through its magnitude alone.
For $\Delta\mu\neq0$, Appendix~\ref{app:calibration} further shows
that $b^*$ maximizes the squared reference separation normalized by
a regularized diagonal variation surrogate, up to nonzero rescaling.
This provides a fixed-statistic characterization of the calibration
rule, not an auxiliary objective optimized by the actor.

\paragraph{When does visual relevance influence credit?}
With $D=\operatorname{diag}(D_h,d_\rho)$,
Eq.~\eqref{eq:calibration} decomposes as
\begin{equation}
    \boxed{
        e(z)
        =
        e_h(v)
        +
        2
        \underbrace{
            \frac{\Delta\mu_\rho}{d_\rho}
        }_{\text{outcome-dependent coefficient}}
        (\rho-\overline\mu_\rho).
    }
    \label{eq:visual-decomposition}
\end{equation}
For fixed statistics, the direct visual term vanishes when
$\Delta\mu_\rho=0$ and reverses its dependence on $\rho$ when the
contrast changes sign.
We use the signed evidence $c_i e(z_{it})$ to measure agreement with
the response's update direction.
Thus, larger relative visual affinity receive no unconditional
positive credit bonus.
Appendix~\ref{app:visual} derives the conditional sensitivities of the
evidence and raw coefficients and distinguishes them from causal
attribution. This motivates separating visual-relevance measurement
from credit assignment.

\subsection{Estimating Relative Visual Evidence}
\label{sec:estimator}

\paragraph{Relative Visual Encoding.}
Let $h_{it}\in\mathbb R^H$ be the state at a fixed intermediate layer
aligned with the prediction of $y_{it}$.
Let $H_{is}$ denote states at the same layer over the full input
sequence. Image positions $\mathcal V_i$ and non-image prompt
positions $\mathcal C_i$ define
\begin{equation}
    g_i^{\mathrm{img}}
    =
    \frac{1}{|\mathcal V_i|}
    \sum_{s\in\mathcal V_i}H_{is},
    \qquad
    g_i^{\mathrm{ctx}}
    =
    \frac{1}{|\mathcal C_i|}
    \sum_{s\in\mathcal C_i}H_{is}.
    \label{eq:prototypes}
\end{equation}
The context prototype is obtained from the original multi-modal input,
not a separate text-only forward.
Let $S_i=\operatorname{diag}(\sigma_i)^{-1}$, where $\sigma_i$ contains
floored response-wise population standard deviations of $h_{it}$.
Applying the same scaling to the predictive state and both prototypes,
RVE computes
\begin{equation}
    \rho_{it}
    =
    \cos(S_i h_{it},S_i g_i^{\mathrm{img}})
    -
    \cos(S_i h_{it},S_i g_i^{\mathrm{ctx}}).
    \label{eq:rve}
\end{equation}
This measures image affinity relative to prompt context.
The mean used to estimate $\sigma_i$ is not subtracted in the cosine
calculation.
Appendix~\ref{app:features} gives the estimator and its metric
interpretation.

We retain sampled-token prediction information with an
uncertainty-scaled feature
\begin{equation}
    v_{it}=(1-p_{it})h_{it},
    \qquad
    z_{it}=[v_{it};\rho_{it}],
    \label{eq:joint-feature}
\end{equation}
where
$p_{it}=\pi_\theta(y_{it}\mid x_i,y_{i,<t})$
is evaluated by the current actor.
Only $v_{it}$ is multiplied by $1-p_{it}$.
Using a state from the selected intermediate layer makes $v_{it}$ a
gradient-inspired predictive feature rather than an exact parameter
gradient (Appendix~\ref{app:proxy}).

\paragraph{Outcome-conditioned Evidence Calibration.}
Given $\mu_c^{(0)}$ and $M_c$ from
Eq.~\eqref{eq:reference-initialization},
OEC estimates advantage-weighted within-side variation
over the actor mini-batch:
\begin{equation}
    q_k
    =
    \sum_{c\in\{-1,+1\}}
    \frac{1}{M_c}
    \sum_{(i,t)\in\mathcal T_c}
    |A_i|
    \left(z_{itk}-\mu_{c,k}^{(0)}\right)^2.
    \label{eq:moments}
\end{equation}
For all $H+1$ coordinates, let
$\overline q=H^{-1}\sum_{k=1}^{H}q_k$.
We construct the diagonal scale matrix
$D=\operatorname{diag}(d)$ using shrinkage and a variance floor:
\begin{equation}
    d_k = \epsilon_D + \max\left\{
    \begin{aligned}
        &(1-\eta)q_k+\eta\overline q,\\[-2pt]
        &r_{\min}\overline q
    \end{aligned}
    \right\}.
    \label{eq:metric}
\end{equation}
Here $0\leq\eta\leq1$, $r_{\min}\geq0$, and $\epsilon_D>0$,
ensuring $d_k\geq\epsilon_D$ and $D\succ0$.
The visual coordinate is excluded from the shrinkage target
$\overline q$ but receives the same regularization.
This is a diagonal scale estimate, not a full covariance matrix.

We perform one soft reference refinement.
Let $e^{(0)}$ use the initial references in Eq.~\eqref{eq:evidence},
and define $\tau_c$ as the unweighted population standard deviation
of $c e^{(0)}(z_{it})$ over $\mathcal T_c$, with a strictly positive
floor. We then compute
\begin{equation}
\begin{aligned}
    u_{it}^{(0)}
    &=
    \operatorname{sigmoid}\!\left(
        \frac{c_i e^{(0)}(z_{it})}{\tau_{c_i}}
    \right),\\[-2pt]
    \mu_c^*
    &=
    \frac{
        \sum_{\mathcal T_c}|A_i|u_{it}^{(0)}z_{it}
    }{
        \sum_{\mathcal T_c}|A_i|u_{it}^{(0)}
    }.
\end{aligned}
\label{eq:refinement}
\end{equation}
The final score $e^*$ uses $\mu_+^*,\mu_-^*$ and the same $D$;
the final coefficient mapping also retains the pre-refinement $\tau_c$.
This makes the visual coefficient in
Eq.~\eqref{eq:visual-decomposition} outcome-dependent after
refinement as well.
Appendix~\ref{app:refinement} specifies the ordering and derives the
soft assignment and reference update.

\subsection{Credit Allocation and Policy Optimization}
\label{sec:optimization}

\paragraph{Normalized Credit Allocation.}
For $0<\lambda_{\min}\leq\lambda_{\max}$, NCA maps signed evidence
to positive raw coefficients and applies token-mean normalization
over the actor mini-batch:
\begin{equation}
\begin{aligned}
    \widehat w_{it}
    &= \lambda_{\min}
    +(\lambda_{\max}-\lambda_{\min})\\[-2pt]
    &\quad\cdot\operatorname{sigmoid}\!\left(
        \frac{c_i e^*(z_{it})}{\tau_{c_i}}
    \right),\\[-1pt]
    Z
    &= \frac{1}{N}\sum_{i,t}\widehat w_{it},
    \quad w_{it}=\frac{\widehat w_{it}}{Z},\\[-1pt]
    A'_{it}
    &= \operatorname{sg}(w_{it})A_i.
\end{aligned}
\label{eq:weights}
\end{equation}
For $A_i=0$, we instead set
$\widehat w_{it}=\lambda_{\min}$ and include it in $Z$.
Positive coefficients preserve advantage signs.
Appendix~\ref{app:bounds} gives their final bounds.
Normalization is global rather than per response, so the estimator
can adjust both relative response strength and allocation within
a response. It fixes the global mean coefficient, but does not
constrain the norm of the resulting policy gradient.

We apply these coefficients through GSPO-token.
With
$L_{it}(\theta)=\log\pi_\theta(y_{it}\mid x_i,y_{i,<t})$,
define
\begin{equation}
    \begin{aligned}
        s_i(\theta)
        &=
        \exp\left[
            \frac1{T_i}
            \sum_t
            \bigl(
                L_{it}(\theta)
                -
                L_{it}(\theta_{\mathrm{old}})
            \bigr)
        \right],
        \\
        \widetilde s_{it}(\theta)
        &=
        \operatorname{sg}(s_i(\theta))
        \exp\bigl(
            L_{it}(\theta)
            -
            \operatorname{sg}(L_{it}(\theta))
        \bigr).
    \end{aligned}
    \label{eq:ratio}
\end{equation}
Here $\operatorname{sg}$ stops gradients.
Let $\widetilde s_{it}^{\,c}
=\operatorname{clip}(
\widetilde s_{it},1-\epsilon_l,1+\epsilon_h)$.
The surrogate to maximize is
\begin{equation}
\mathcal J_{\mathrm{OVC}}(\theta)
=\frac{1}{B}\sum_i\frac{1}{T_i}\sum_t
\min\bigl\{
\widetilde s_{it}A'_{it},
\widetilde s_{it}^{\,c}A'_{it}
\bigr\}.
\label{eq:objective}
\end{equation}
Although $\widetilde s_{it}=s_i$ numerically, its derivative remains
token-specific.
Appendices~\ref{app:gradient} and~\ref{app:allocation} derive this
gradient and separate response-level weighting from token allocation.
No gradients pass through the credit estimator.
The policy loss can still update the entire actor, including its
visual encoder.
OVC reuses the sampled responses and adds no learnable head or
inference-time module.
Algorithm~\ref{alg:ovc} and Appendix~\ref{app:implementation}
describe the actual multi-pass training computation.
\section{Experiment}
\subsection{Experimental Setup}

\paragraph{Models and training data.}
We evaluate Qwen3-VL-4B-Instruct,
Qwen3-VL-8B-Instruct~\citep{qwen3}, and
Qwen3.5-9B~\citep{qwen35}.
For main results, the Qwen3-VL models are trained on
ViRL39K~\citep{Vl-rethinker}, while Qwen3.5-9B uses
Math77K, the 77,135 mathematical subset of
DeepVision-103K~\citep{deepvision}.
Within each main-comparison setting, all methods share
the same instruction-tuned initialization.

\paragraph{Benchmarks.}
For main results, we report
MMMU~\citep{mmmu}, MathVerse~\citep{mathverse},
MathVision~\citep{mathvision}, We-Math~\citep{wemath},
MathVista~\citep{mathvista}, and MMK12-val~\citep{Mm-eureka}.
The first five use LMMs-Eval~\citep{lmmseval} 
with matched decoding and scoring within each model size. 
MMK12-val follows the fixed training-time validation protocol. 
Full settings are given in Appendix~\ref{app:evaluation}.

\subsection{Experimental Results}

\label{sec:comparison_methods}


\subsubsection{Main Results}

\label{sec:main_results}

\providecommand{\avggain}[1]{%
    \textsuperscript{\textcolor{red}{+#1}}%
}

\begin{table*}[t]
    \centering
    \small
    \setlength{\tabcolsep}{3pt}
    \renewcommand{\arraystretch}{0.88}
    \setlength{\aboverulesep}{1pt}
    \setlength{\belowrulesep}{1pt}

    \caption{
        \textbf{Main results on six multimodal benchmarks.}
        Scores (\%) are averaged over four evaluation 
        sampling seeds at fixed checkpoints.
        Bold and underlined values indicate the best and
        second-best distinct scores within each model group.
        Avg.\ is the six-benchmark mean; red superscripts
        denote gains over Base.
        Qwen3.5-9B uses larger-scale Math77K training data.
        OOM denotes out of memory.
    }
    \label{tab:main_results}

    \begin{tabularx}{\textwidth}{
        @{}l
        *{7}{>{\centering\arraybackslash}X}
        @{}
    }
        \toprule
        \textbf{Method}
        & \textbf{MMMU}
        & \textbf{MathVision}
        & \textbf{We-Math}
        & \textbf{MathVista}
        & \textbf{MMK12-val}
        & \textbf{MathVerse}
        & \textbf{Avg.} \\
        \midrule

        \multicolumn{8}{@{}l}{
            \textit{Qwen3-VL-4B}
        } \\

        Base
        & \underline{46.56} & 21.71 & 75.52
        & 73.83 & 54.73 & 63.12 & 55.91 \\

        GSPO~\citep{GSPO}
        & 44.67 & 24.58 & \underline{76.52}
        & \textbf{74.83} & 67.14 & 64.29
        & 58.67\avggain{2.76} \\

        DelTA~\citep{DelTA}
        & 44.33 & \underline{25.00} & 76.00
        & 74.00 & \textbf{74.58} & \underline{64.67}
        & \underline{59.76}\avggain{3.85} \\

        VPPO~\citep{VPPO}
        & 45.21 & 23.12 & 76.43
        & 74.19 & 69.23 & 64.23
        & 58.74\avggain{2.83} \\

        AT-RL~\citep{AT-RL}
        & 43.52 & 22.78 & 76.33
        & 74.22 & \underline{70.22} & 63.45
        & 58.42\avggain{2.51} \\

        \rowcolor{blue!10}
        \textbf{GSPO+OVC (Ours)}
        & \textbf{46.83} & \textbf{26.64}
        & \textbf{77.05} & \underline{74.37}
        & \textbf{74.58} & \textbf{64.80}
        & \textbf{60.71}\avggain{4.80} \\

        \midrule

        \multicolumn{8}{@{}l}{
            \textit{Qwen3-VL-8B}
        } \\

        Base
        & 48.78 & 24.67 & 80.23
        & 74.40 & 64.46 & 62.84
        & 59.23 \\

        GSPO~\citep{GSPO}
        & 44.56 & 27.30 & 83.33
        & 76.10 & 74.73 & 63.53
        & 61.59\avggain{2.36} \\

        DelTA~\citep{DelTA}
        & 49.21 & 24.01 & \textbf{85.40}
        & \textbf{77.03} & \underline{77.06}
        & \underline{64.82}
        & \underline{62.92}\avggain{3.69} \\

        VPPO~\citep{VPPO}
        & \underline{49.43} & \underline{27.69}
        & \underline{84.36} & 76.21
        & 75.29 & 64.09
        & 62.85\avggain{3.62} \\

        AT-RL~\citep{AT-RL}
        & \texttt{OOM} & \texttt{OOM}
        & \texttt{OOM} & \texttt{OOM}
        & \texttt{OOM} & \texttt{OOM}
        & \texttt{OOM} \\

        \rowcolor{blue!10}
        \textbf{GSPO+OVC (Ours)}
        & \textbf{49.56} & \textbf{29.93}
        & 83.45 & \underline{76.97}
        & \textbf{78.94} & \textbf{65.05}
        & \textbf{63.98}\avggain{4.75} \\

        \midrule

        \rowcolor{teal!8}
        \multicolumn{8}{@{}l}{
            \textit{Qwen3.5-9B}
            \quad
            \textcolor{teal!65!black}{
                \textit{(Larger-scale training: Math77K)}
            }
        } \\

        Base
        & 45.67 & \underline{36.18}
        & \underline{87.24} & \underline{79.20}
        & 79.09 & \underline{77.36}
        & 67.46 \\

        GSPO~\citep{GSPO}
        & \underline{46.78} & 34.87
        & 87.18 & \underline{79.20}
        & \underline{87.74} & \textbf{77.74}
        & \underline{68.92}\avggain{1.46} \\

        \rowcolor{blue!10}
        \textbf{GSPO+OVC (Ours)}
        & \textbf{50.33} & \textbf{40.46}
        & \textbf{88.51} & \textbf{80.30}
        & \textbf{88.28} & 76.80
        & \textbf{70.78}\avggain{3.32} \\

        \bottomrule
    \end{tabularx}
\end{table*}

We compare with the instruction-tuned Base model, GSPO~\citep{GSPO},
DelTA~\citep{DelTA}, VPPO~\citep{VPPO}, and AT-RL~\citep{AT-RL}.
OVC applies its detached weights through GSPO-token; 
uniform weights recover the GSPO gradient (Appendix~\ref{app:gradient}).

\paragraph{Consistent gains across model scales.}
Table~\ref{tab:main_results} shows that OVC improves the six-task
average over GSPO by \ovcdelta{+2.04}, \ovcdelta{+2.39}, 
and \ovcdelta{+1.86} points on 4B, 8B, and 9B, respectively, 
and exceeds Base by \ovcdelta{+4.80}, \ovcdelta{+4.75}, 
and \ovcdelta{+3.32} points.
OVC also achieves the highest average among all compared RL methods
on both Qwen3-VL scales, with the gain extending to Qwen3.5-9B
trained on the larger-scale Math77K corpus.

\paragraph{Performance across multimodal benchmarks.}
The largest gains over GSPO occur on MMK12-val, MathVision, and MMMU:
\ovcdelta{+7.44}/\ovcdelta{+4.21} points on MMK12-val for 4B/8B,
\ovcdelta{+5.59} on MathVision for 9B, and
\ovcdelta{+2.16}/ \ovcdelta{+5.00}/\ovcdelta{+3.55} on MMMU.
OVC also reverses the Base-to-GSPO drop on MMMU for 4B/8B,
reaching 46.83/49.56 versus Base scores of 46.56/48.78.
On 9B, the only regression relative to GSPO is MathVerse
(\ovcdelta{-0.94} points).

\paragraph{Comparison with non-visual discriminative credit.}
OVC exceeds DelTA's six-task average by \ovcdelta{+0.95} points 
on 4B and \ovcdelta{+1.06} on 8B. 
It leads DelTA on five benchmarks and ties on MMK12-val at 4B. 
At 8B, its largest gains are MathVision (\ovcdelta{+5.92}) and 
MMK12-val (\ovcdelta{+1.88}), while DelTA remains stronger on 
We-Math and nearly tied on MathVista.

\begin{table*}[!htbp]
    \centering

    \definecolor{ovcGroup}{HTML}{F2F4F7}
    \definecolor{ovcAccent}{HTML}{314F70}
    \definecolor{ovcDrop}{HTML}{A65A5A}

    \small
    \renewcommand{\arraystretch}{0.85}
    \setlength{\tabcolsep}{3pt}
    \setlength{\aboverulesep}{1pt}
    \setlength{\belowrulesep}{1pt}

    \caption{
        \textbf{Ablation study of OVC on Qwen3-VL-8B.}
        Base and GSPO are from Table~\ref{tab:main_results}.
        Scores are percentages; Avg.\ denotes the four-benchmark mean.
        $\Delta$ indicates the change from full OVC in percentage points.
        Best scores are bold.
        See Appendix~\ref{app:ablation-configs} for details.
    }
    \label{tab:ovc_ablation}

    \begin{tabularx}{\textwidth}{
        @{}
        >{\raggedright\arraybackslash}X
        *{5}{>{\centering\arraybackslash}p{1.7cm}}
        >{\raggedleft\arraybackslash
          \leavevmode\color{ovcDrop}}p{1.0cm}
        @{}
    }
        \toprule
        \textbf{Configuration}
        & \textbf{MathVision}
        & \textbf{MathVista}
        & \textbf{MathVerse}
        & \textbf{MMK12-val}
        & \textbf{Avg.}
        & \multicolumn{1}{r@{}}{\boldmath$\Delta$} \\
        \midrule

        Base
        & 24.67 & 74.40 & 62.84 & 64.46
        & 56.59 & $-6.13$ \\

        GSPO (uniform credit)
        & 27.30 & 76.10 & 63.53 & 74.73
        & 60.42 & $-2.30$ \\

        \rowcolor{blue!8}
        \textbf{OVC (full)}
        & \textbf{29.93}
        & \textbf{76.97}
        & \textbf{65.05}
        & \textbf{78.94}
        & \textbf{62.72}
        & \textcolor{black}{---} \\

        \addlinespace[1pt]
        \rowcolor{ovcGroup}
        \multicolumn{7}{@{}l@{}}{
            \textcolor{ovcAccent}{
                \textbf{(a) Relative visual encoding}
            }
        } \\

        No visual coordinate
        & 28.21 & 74.97 & 63.46 & 75.43
        & 60.52 & $-2.20$ \\

        Absolute image affinity
        & 28.03 & 75.63 & 63.98 & 76.31
        & 60.99 & $-1.73$ \\

        Separate affinity coordinates
        & 28.89 & 75.59 & 63.83 & 74.20
        & 60.63 & $-2.09$ \\

        No response-wise scaling ($S_i=I$)
        & 28.89 & 74.99 & 63.79 & 75.32
        & 60.75 & $-1.97$ \\

        \addlinespace[1pt]
        \rowcolor{ovcGroup}
        \multicolumn{7}{@{}l@{}}{
            \textcolor{ovcAccent}{
                \textbf{(b) Within-side variance calibration}
            }
        } \\

        Identity metric ($D=I$)
        & 28.12 & 75.34 & 64.25 & 76.21
        & 60.98 & $-1.74$ \\

        Total-variance calibration
        & 28.09 & 75.91 & 64.12 & 74.33
        & 60.61 & $-2.11$ \\

        No visual coordinate + $D=I$
        & 28.64 & 75.23 & 63.67 & 75.41
        & 60.74 & $-1.98$ \\

        \addlinespace[1pt]
        \rowcolor{ovcGroup}
        \multicolumn{7}{@{}l@{}}{
            \textcolor{ovcAccent}{
                \textbf{(c) Token-specific weight allocation}
            }
        } \\

        Response-mean weights
        & 28.06 & 75.39 & 64.88 & 76.32
        & 61.16 & $-1.56$ \\

        Within-response permutation
        & 27.54 & 75.14 & 62.43 & 74.28
        & 59.85 & $-2.87$ \\

        Affinity-ranked weights
        & 28.57 & 76.86 & 64.39 & 76.34
        & 61.54 & $-1.18$ \\

        \bottomrule
    \end{tabularx}
\end{table*}

\paragraph{Gains from strong initializations.}
These gains are measured on strong,
already post-trained models.
Prior exposure to related visual-mathematical problems may reduce
the marginal gains from additional RL, although sample-level overlap
with the original training corpora remains unverified.
Under matched initializations, RL data, and rollout configurations,
OVC's gains are consistent with more effective use of existing
representations and outcome feedback through finer-grained credit
allocation.
These results suggest that improved credit assignment remains
beneficial even with substantial task-relevant competence.

\subsubsection{Ablation Studies}
\label{sec:ablation}

Table~\ref{tab:ovc_ablation} examines OVC on Qwen3-VL-8B across
relative visual encoding, within-side variance calibration, and
token-specific weight allocation.
Full OVC achieves a four-task mean of 62.72, exceeding the
uniform-credit GSPO reference (60.42) by 2.30 percentage points.
Configuration definitions are provided in
Appendix~\ref{app:ablation-configs}.

\paragraph{Relative visual encoding.}
Removing the visual coordinate, using absolute image affinity, or
using separate image/context coordinates lowers the four-task mean
by 2.20, 1.73, and 2.09 points, respectively.
Separate coordinates retain both affinity measurements yet
underperform their explicit difference, supporting the use of
a context-relative visual representation.
Removing response-wise scaling ($S_i=I$) further reduces the
mean by 1.97 points, indicating that pre-cosine 
geometry contributes to evidence estimation.

\paragraph{Within-side variance calibration.}
The identity metric and total-variance calibration lower the mean
by 1.74 and 2.11 points, respectively.
Setting $D=I$ retains the advantage-side reference contrast but
removes variance scaling.
Adding the relative visual coordinate increases the mean from
60.52 to 62.72 with within-side calibration, compared with a
smaller increase from 60.74 to 60.98 under the identity metric.
Conversely, within-side calibration slightly reduces performance
without the visual coordinate (60.74 to 60.52), but substantially
improves it with the coordinate (60.98 to 62.72).
These results yield a positive interaction of 1.96 points,
suggesting that relative visual encoding and variance calibration
are complementary rather than independently sufficient.

\paragraph{Token-specific weight allocation.}
Response-mean, within-response permutation, and affinity-ranked
weights reduce the mean by 1.56, 2.87, and 1.18 points, respectively.
Response averaging preserves each response's mean coefficient,
while permutation additionally preserves its complete weight
multiset but disrupts token--weight correspondence.
Their deficits support the importance of token-specific allocation
beyond response-level reweighting or nonuniform coefficients alone.
Affinity-ranked weights also underperform full OVC, suggesting
that directly favoring higher visual affinity cannot reproduce
outcome-conditioned allocation.
Although 9/10 ablations remain numerically above GSPO,
none matches full OVC, with within-response permutation even
falling below the uniform-credit reference.

\subsubsection{Chart-centric Reasoning}
\label{sec:chart_reasoning}
We additionally evaluate OVC on chart-centric reasoning with
Qwen3-VL-4B-Instruct. 
After SFT on 100K ChartVerse examples~\citep{chartverse}, 
the resulting model serves as Base.
GSPO and GSPO+OVC start from Base and train
for 72 rollout steps on the same 18,280-question subset of ChartVerse-RL-40K.
Training and evaluation details are provided in
Appendices~\ref{app:chart_training} and~\ref{app:chart_eval}.


\begin{table}[htbp]
    \centering
    \definecolor{ovcPurple}{HTML}{F2ECFA}
    \definecolor{ovcGain}{HTML}{B24A4A}

    \caption{
        \textbf{Chart-centric reasoning on Qwen3-VL-4B.}
        Scores (\%) are averaged over three evaluation
        sampling seeds.
        Avg.\ is the six-benchmark mean.
        Bold and underlined values denote the best
        and second-best results.
        Red superscripts indicate gains over Base.
    }
    \label{tab:chart_results}

    \footnotesize
    \setlength{\tabcolsep}{2.5pt}
    \renewcommand{\arraystretch}{0.95}

    \resizebox{\columnwidth}{!}{%
    \begin{tabular}{@{}lccccccc@{}}
        \toprule
        \textbf{Method}
        & \shortstack{\textbf{Chart}\\\textbf{QA-Pro}}
        & \shortstack{\textbf{Char}\\\textbf{Xiv}}
        & \shortstack{\textbf{Chart}\\\textbf{Museum}}
        & \shortstack{\textbf{ChartX-}\\\textbf{QA}}
        & \shortstack{\textbf{Evo}\\\textbf{Chart-QA}}
        & \shortstack{\textbf{Chart}\\\textbf{Bench}}
        & \textbf{Avg.} \\
        \midrule

        Base
        & 37.69
        & 62.66
        & 36.83
        & 63.69
        & 47.39
        & 62.08
        & 51.72 \\

        GSPO
        & \underline{40.56}
        & \underline{63.47}
        & \underline{41.57}
        & \textbf{66.55}
        & \underline{47.60}
        & \underline{65.03}
        & \underline{54.13}
          \textsuperscript{\textcolor{ovcGain}{+2.41}} \\

        \rowcolor{ovcPurple}
        \textbf{OVC (Ours)}
        & \textbf{44.18}
        & \textbf{64.77}
        & \textbf{43.42}
        & \underline{66.30}
        & \textbf{48.28}
        & \textbf{65.18}
        & \textbf{55.36}
          \textsuperscript{\textcolor{ovcGain}{+3.64}} \\

        \bottomrule
    \end{tabular}%
    }

    
\end{table}
As shown in Table~\ref{tab:chart_results}, 
OVC achieves the best performance on 5/6 
benchmarks~\citep{chartqapro, charxiv, chartmuseum, chartx, evochart, chartbench}, 
improving the average score from 54.13 to 55.36 (+1.23 points) 
over GSPO and exceeding Base by 3.64 points. 
The largest gains over GSPO occur on ChartQA-Pro (+3.62) and ChartMuseum (+1.85), 
while ChartX-QA shows a slight decrease (-0.25). 
These results support the applicability of OVC to
chart-focused reinforcement learning under matched
initialization, training data, and rollout budgets.

\subsubsection{Extension to Other Policy Optimizers}
\label{sec:optimizer_extension}
To examine whether OVC is applicable beyond GSPO,
we extend it to DAPO and GRPO, which also use
shared response-level advantages for token updates.
Experiments are conducted on Qwen2.5-VL-7B~\citep{qwen25vl}
trained with ViRL39K.


\begin{table}[!htbp]
    \centering
    \definecolor{ovcPurple}{HTML}{F2ECFA}

    \footnotesize
    \setlength{\tabcolsep}{2.5pt}
    \renewcommand{\arraystretch}{0.95}

    \caption{
        \textbf{Extending OVC to alternative policy optimizers.} 
        Results on Qwen2.5-VL-7B trained on ViRL39K.
        Avg.\ is the mean across six benchmarks. 
        Bold marks the better result in each pair.
    }
    \label{tab:ovc_other_optimizers}

    \resizebox{\columnwidth}{!}{%
    \begin{tabular}{@{}lccccccc@{}}
        \toprule

        \textbf{Method}
        & \textbf{MMMU}
        & \shortstack{\textbf{Math}\\\textbf{Vision}}
        & \shortstack{\textbf{We-}\\\textbf{Math}}
        & \shortstack{\textbf{Math}\\\textbf{Vista}}
        & \shortstack{\textbf{MMK12}\\\textbf{-val}}
        & \shortstack{\textbf{Math}\\\textbf{Verse}}
        & \textbf{Avg.} \\

        \midrule

        DAPO
        & 50.00
        & 14.80
        & 69.54
        & 70.67
        & \textbf{76.88}
        & 45.58
        & 54.58 \\

        \rowcolor{ovcPurple}
        \textbf{DAPO+OVC}
        & \textbf{52.78}
        & \textbf{20.39}
        & \textbf{70.92}
        & \textbf{71.30}
        & 76.58
        & \textbf{45.89}
        & \textbf{56.31} \\

        \midrule

        GRPO
        & 51.11
        & 14.80
        & 69.37
        & 70.20
        & 74.88
        & 45.00
        & 54.23 \\

        \rowcolor{ovcPurple}
        \textbf{GRPO+OVC}
        & \textbf{52.22}
        & \textbf{15.13}
        & \textbf{70.63}
        & \textbf{70.27}
        & \textbf{75.50}
        & \textbf{46.17}
        & \textbf{54.99} \\

        \bottomrule
    \end{tabular}%
    }
\end{table}

As shown in Table~\ref{tab:ovc_other_optimizers}, 
OVC increases the average of 6 tasks by 1.73 and
0.76 points for DAPO and GRPO, respectively.
Improvements are observed on 5/6 benchmarks
for DAPO and all six for GRPO.
These results illustrate the applicability of OVC
to different policy optimization objectives,
rather than restricting token-level credit
reweighting to GSPO.

\subsubsection{Zero-shot Generalization}
\label{sec:zero_shot_generalization}

To assess transfer beyond the main evaluation suite, we evaluate
two groups of tasks: chart understanding and VQA, and scientific
and spatial reasoning. 
These transfer evaluations use the ViRL39K- and
Math77K-trained checkpoints from the main comparisons
and their corresponding instruction-tuned Base checkpoints.
\begin{figure}[!htbp]
    \centering
    \includegraphics[
      width=\linewidth,
      trim=2pt 24pt 0 11pt,
      clip
    ]{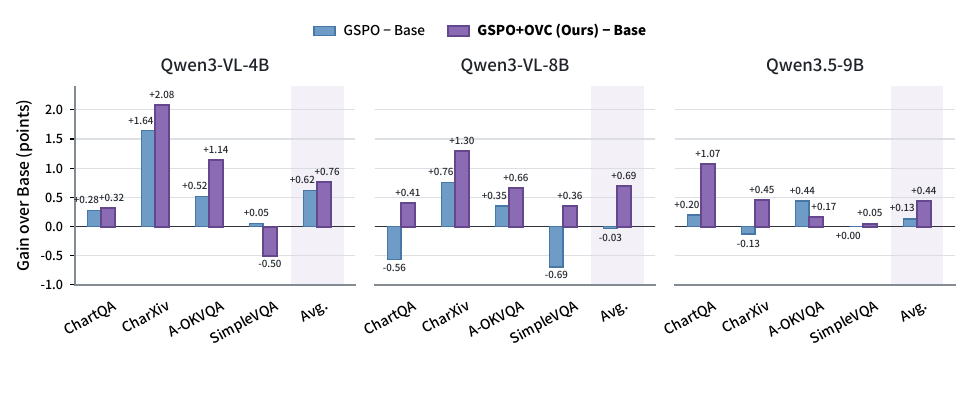}
    \caption{\textbf{Zero-shot transfer gains over Base.}
        Bars show score changes relative to the corresponding instruction-tuned
        Base model. Avg. is the mean over four VQA benchmarks.}
    \label{fig:vqa}
\end{figure}

\begin{figure*}[t]
  \centering
  \includegraphics[
      width=\linewidth,
      trim=0 9pt 0 6pt,
      clip
    ]{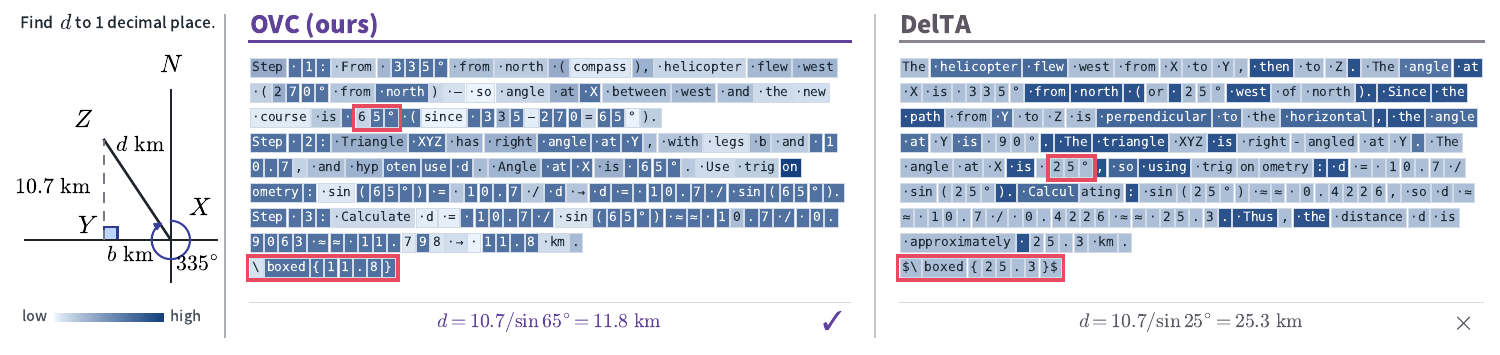}
  \caption{\textbf{Qualitative visualization of within-response
token-credit patterns.}
OVC correctly converts the bearing to the $65^\circ$ triangle angle
and obtains $d=11.8$ km, whereas DelTA first errs by using $25^\circ$
as the triangle angle and obtains $25.3$ km. Colors indicate the final
normalized token-credit weights $w_{it}$, recomputed offline from the
corresponding checkpoints. Red boxes mark the key angle conversions
and final answers. The OVC and DelTA responses have positive and
negative group-relative advantages, respectively.}
  \label{fig:credit_visualization}
\end{figure*}

\paragraph{General-purpose VQA.}

We evaluate ChartQA~\citep{chartqa},
CharXiv~\citep{charxiv},
A-OKVQA~\citep{a-okvqa}, and
SimpleVQA~\citep{simplevqa} across three model scales.
As shown in Figure~\ref{fig:vqa}, OVC improves
the four-task mean over GSPO by 0.14, 0.72, and
0.31 points on 4B, 8B, and 9B, respectively,
while consistently exceeding the corresponding Base.
The largest improvement over GSPO occurs on 8B,
where OVC achieves a 0.69-point gain over Base,
compared with a 0.03-point decrease for GSPO.
Benchmark descriptions, score aggregation, and absolute results
are provided in Appendix~\ref{app:vqa_generalization}.

\paragraph{Scientific and spatial reasoning.}

\begin{table}[!htbp]
    \centering
    \caption{
        \textbf{Zero-shot transfer to scientific and spatial benchmarks.}
        GSPO and OVC are trained on Math77K using Qwen3.5-9B.
        Avg.\ denotes the mean score across five benchmarks.
        The best results are highlighted in bold.
    }
    \label{tab:zero_shot_transfer}

    \scriptsize
    \setlength{\tabcolsep}{2.5pt}
    \renewcommand{\arraystretch}{1.10}

    \resizebox{\columnwidth}{!}{%
    \begin{tabular}{lcccccc}
        \toprule
        & \multicolumn{3}{c}{\textit{Science \& Domain-Specific}}
        & \multicolumn{2}{c}{\textit{Spatial \& Embodied AI}}
        & \\
        \cmidrule(lr){2-4}
        \cmidrule(lr){5-6}

        \textbf{Method}
        & \textbf{ScienceQA}
        & \textbf{RxnBench}
        & \textbf{SUPERChem}
        & \textbf{EmbSpatial}
        & \textbf{ERQA}
        & \textbf{Avg.} \\
        \midrule

        Qwen3.5-9B
            & 92.88 & 87.08 & 10.20
            & 73.19 & 44.50 & 61.57 \\

        GSPO
            & 93.15 & 88.32 & 12.30
            & 75.80 & 45.20 & 62.95 \\

        \rowcolor{blue!8}
        \textbf{OVC (Ours)}
            & \textbf{93.89}
            & \textbf{89.05}
            & \textbf{13.40}
            & \textbf{77.45}
            & \textbf{45.75}
            & \textbf{63.91} \\

        \bottomrule
    \end{tabular}%
    }
\end{table}

We further examine cross-domain transfer using
Qwen3.5-9B trained on Math77K, evaluating five
benchmarks~\citep{ScienceQA, RxnBench, SuperChem, EmbSpatial, ERQA}
covering scientific knowledge and spatial reasoning.
As shown in Table~\ref{tab:zero_shot_transfer},
OVC outperforms GSPO on all five benchmarks,
raising the average score from 62.95 to 63.91.
The largest gains are observed on EmbSpatial
(+1.65 points) and SuperChem (+1.10 points).
OVC also exceeds the Base model by 2.34 points
on average.
Together, these results suggest that the benefits
of outcome-conditioned token reweighting extend
beyond the visual-mathematical evaluation domain.
Benchmark descriptions and score aggregation are provided in
Appendix~\ref{app:science_spatial_generalization}.

\subsection{Visualizations}
\label{sec:visualizations}

\paragraph{Distribution of token credit weights.}
\label{sec:weight_distribution}

Figure~\ref{fig:credit_visualization} provides a qualitative
visualization of the token-credit patterns produced by OVC and DelTA.
OVC correctly converts the bearing to the $65^\circ$ interior angle
and obtains $d=11.8$ km. Its decisive reasoning span receives
relatively high credit. In contrast, DelTA uses $25^\circ$ as the
triangle angle and obtains $25.3$ km, while assigning relatively low
credit to the first erroneous reasoning span. This example illustrates
how OVC can better align token-level credit with visually grounded
reasoning steps.

\begin{figure}[!htbp]
    \centering
    \includegraphics[
      width=\linewidth,
      trim=1pt 7pt 0 1pt,
      clip
    ]{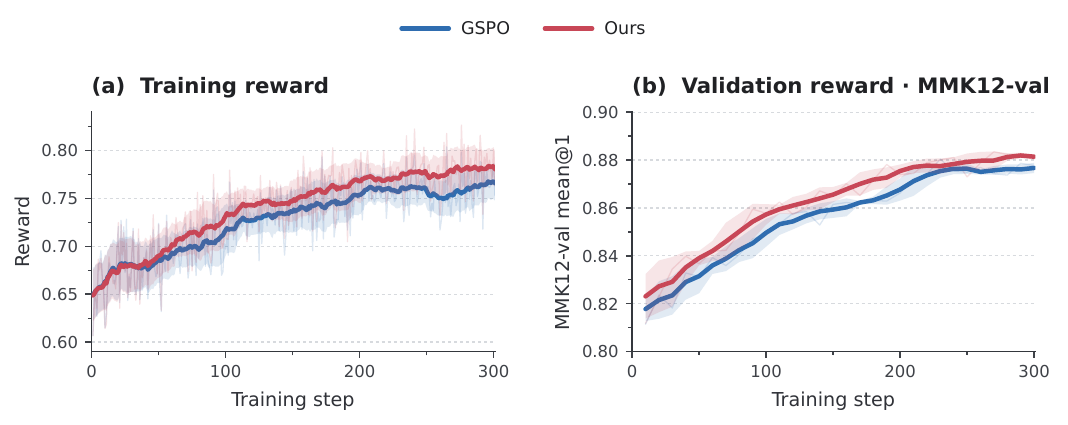}
    \caption{Training dynamics on Qwen3.5-9B.}
    \label{fig:training_curve}
\end{figure}
\paragraph{Training dynamics.}

Figure~\ref{fig:training_curve} compares OVC and GSPO during
Qwen3.5-9B training on Math77K. OVC generally achieves higher
training reward and MMK12-val performance over the training trajectory.
The consistent advantage in both optimization reward and validation
performance provides further evidence for the effectiveness of OVC
beyond the final-checkpoint comparison.

\begin{figure}[!htbp]
    \centering
    \includegraphics[
      width=\linewidth,
      trim=0 1pt 0 0,
      clip
    ]{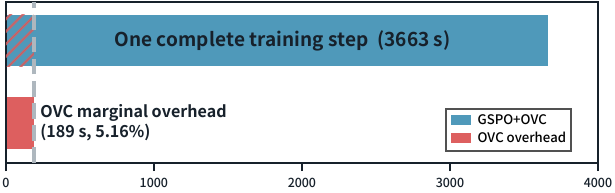}
    \caption{Training-time cost breakdown of OVC.}
    \label{fig:ovc_overhead}
\end{figure}

\subsection{Efficiency}
\label{sec:efficiency}

Figure~\ref{fig:ovc_overhead} shows the training-time cost breakdown. 
The OVC-specific credit-estimation passes account for
189\,s, or 5.16\% of the 3663\,s total step time. OVC adds no
learnable parameters or inference-time computation; its additional
training cost comes from four no-gradient statistics passes.
Implementation details are given in Appendix~\ref{app:implementation}.







\section{Conclusion}
We introduced OVC to address a central distinction in multimodal RLVR: visual association does not, by itself, determine learning credit under response-level feedback. OVC measures visual affinity relative to context, combines it with predictive features, and calibrates their contributions through advantage-side contrasts and regularized within-side variation. The resulting evidence modulates token-specific gradient contributions through positive, normalized weights, without changing the original advantage signs or sequence-level clipping. Our analysis further distinguishes response-level reweighting from within-response allocation, clarifying how calibrated coefficients affect individual token gradients. Experiments across three models and six benchmarks show improved average accuracy over GSPO without additional process supervision or inference-time modules. 
Overall, OVC connects visual representations with fine-grained policy
optimization by treating visual association as outcome-conditioned
evidence rather than direct update priority.

{
    \small
    \bibliographystyle{ieeenat_fullname}
    \bibliography{custom}
}

\clearpage
\appendix
\setcounter{page}{1}
\renewcommand{\thepage}{S\arabic{page}}
\maketitlesupplementary
\label{cvpr:supplement-start}
\section{Evaluation protocol}
\label{app:evaluation}

\subsection{Benchmarks and evaluation settings}
\label{app:benchmarks}
For the main comparisons, we use the official MMK12 test
split as a fixed validation set and report the final
scheduled validation score.
We evaluate MMMU, MathVerse, MathVision, We-Math, and
MathVista after training using LMMs-Eval~\citep{lmmseval}.
All compared checkpoints within each main-comparison
setting use the same benchmark splits and evaluation
procedures.
Table~\ref{tab:eval_tasks} summarizes this protocol.
The chart-centric evaluation is described in
Appendix~\ref{app:chart_eval}.

\begin{table*}[t]
\centering
\small
\caption{Benchmark protocol used for external evaluation.}
\label{tab:eval_tasks}
\resizebox{\textwidth}{!}{%
\begin{tabular}{lllll}
\toprule
Benchmark & Split & Samples & Prompting protocol & Scoring protocol \\
\midrule
MMMU & validation & 900 & multiple-choice prompting & deterministic accuracy \\
MathVerse & testmini & 3,940 & zero-shot direct query & DeepSeek V4 correctness judge \\
MathVision & testmini & 304 & original question & DeepSeek V4 correctness judge \\
We-Math & testmini & 1,740 & reasoning with explicit final answer & deterministic Loose accuracy \\
MathVista & testmini & 1,000 & three zero-shot prompting settings & answer extraction + deterministic scoring; mean over settings \\ 
MMK12-val & test & 2,000 & training validation prompt & validation accuracy \\
\bottomrule
\end{tabular}}
\end{table*}

\paragraph{MMMU~\citep{mmmu}.}
We evaluate the 900-example validation split. 
Following the public Qwen protocol, 
all images are placed before the textual query. 
The prompt contains the available hint, question, 
and answer choices when present, and asks the model to 
select the correct option for multiple-choice items. 
Responses are scored by the benchmark's deterministic accuracy evaluator.

\paragraph{MathVerse~\citep{mathverse}.}
We evaluate the 3,940-example testmini split with 
the benchmark's zero-shot query that does not add 
an explicit chain-of-thought instruction or few-shot demonstrations. 
Benchmark-defined termination strings are preserved. 
Responses are judged with DeepSeek V4 as described in 
Sec.~\ref{app:judging}.

\paragraph{MathVision~\citep{mathvision}.}
We evaluate the 304-example testmini split. 
The textual input is the dataset question itself, 
matching the public Qwen evaluation convention, 
with no additional answer-format instruction. 
Responses are judged by the same DeepSeek V4 
correctness protocol used for MathVerse. 

\paragraph{We-Math~\citep{wemath}.}
We evaluate the 1,740-example testmini split. 
The prompt presents the image, question, and answer options 
and requests an explicit reasoning trace 
followed by a clearly delimited final answer. 
We report accuracy from the benchmark's deterministic reasoning evaluator.

\paragraph{MathVista~\citep{mathvista}.}
We evaluate the 1,000-example testmini split under 
three zero-shot prompting settings. 
The \emph{CoT} setting asks for a reasoned answer 
followed by the required final value or option. 
The \emph{Solution} setting explicitly prompts the model to 
produce a solution before the final answer. 
The \emph{Format} setting requests only an answer in the format 
required by the target type, such as an option letter, integer, 
floating-point value with the prescribed precision, or list. 
Image captions and OCR annotations are not provided. 
We report the arithmetic mean over the three settings as the MathVista score.

\paragraph{MMK12-val~\citep{Mm-eureka}.}
MMK12 is an independently collected multimodal K--12 reasoning benchmark 
covering mathematics, physics, chemistry, and biology. 
We use its official 2,000-example test split as 
a fixed validation set during RL training. 
The same split and evaluation procedure are used for 
all compared methods within each model size. 
MMK12 does not participate in policy optimization, 
checkpoint selection, or early stopping. 
In Table~\ref{tab:main_results}, we report the score from 
the final scheduled validation event rather than 
selecting the best score observed during training.


\subsection{Ablation configurations}
\label{app:ablation-configs}

Each variant modifies the indicated representation, metric, or
allocation rule while retaining the remaining OVC procedure.
Statistics are computed from the current actor and rollouts of the
corresponding run, rather than copied from a separately trained
full-OVC model.
Representation and metric variants re-estimate their reference means,
margin temperatures, refined references, and normalization factor.
Within each update, the selected metric and pre-refinement
temperatures remain fixed during the single reference refinement.
All operations use valid response tokens and exclude padding.

\paragraph{Visual representations.}
Define the scaled affinities
\begin{equation}
    a_{it}^{u}
    =\cos(S_i h_{it},S_i g_i^{u}),
    \qquad u\in\{\mathrm{img},\mathrm{ctx}\}.
    \label{eq:ablation-affinities}
\end{equation}
Full OVC uses $\rho_{it}=a_{it}^{\mathrm{img}}-a_{it}^{\mathrm{ctx}}$.
The representation variants replace the joint feature by
\begin{equation}
    \begin{aligned}
        \text{No visual coordinate:}\quad
        &z_{it}=v_{it},\\
        \text{Absolute image affinity:}\quad
        &z_{it}=[v_{it};a_{it}^{\mathrm{img}}],\\
        \text{Separate affinity coordinates:}\quad
        &z_{it}=[v_{it};a_{it}^{\mathrm{img}};a_{it}^{\mathrm{ctx}}].
    \end{aligned}
    \label{eq:ablation-representations}
\end{equation}
The separate-coordinate variant retains the same response-wise
$S_i$ for both affinities and uses an $(H+2)$-dimensional diagonal
metric.
Its shrinkage target remains the mean variation of the first $H$
predictive coordinates; both appended affinity coordinates are
excluded from that target but receive the same regularization
in Eq.~\eqref{eq:metric}.
``Absolute'' means without context subtraction, not an absolute value.
The $S_i=I$ variant removes scaling only from the two cosine
measurements and retains their difference.
In every representation variant, $v_{it}=(1-p_{it})h_{it}$ is unchanged,
and appended affinity coordinates are not multiplied by $1-p_{it}$.

\paragraph{Metric variants.}
The identity variant sets $D=I$ over the entire selected feature
space while retaining advantage-side reference estimation,
soft refinement, and sign-conditioned allocation.
The no-visual identity variant additionally sets $z_{it}=v_{it}$.
For total-variance calibration, the side-normalized weights in
Eq.~\eqref{eq:moments} are retained, but both sides are centered
at their common midpoint:
\begin{equation}
    \begin{aligned}
        \overline\mu^{(0)}
        &=\frac{\mu_+^{(0)}+\mu_-^{(0)}}{2},\\
        q_k^{\mathrm{tot}}
        &=\sum_{c\in\{-1,+1\}}\frac{1}{M_c}
          \sum_{(i,t)\in\mathcal T_c}
          |A_i|\bigl(z_{itk}-\overline\mu_k^{(0)}\bigr)^2\\
        &=q_k+\frac12
          \bigl(\mu_{+,k}^{(0)}-\mu_{-,k}^{(0)}\bigr)^2.
    \end{aligned}
    \label{eq:ablation-total-variance}
\end{equation}
This is twice the total variance of an equally weighted mixture
of the two advantage-side distributions, matching the summed-side
scale of Eq.~\eqref{eq:moments}.
It changes the centering without changing the relative side weights.
The resulting $q_k^{\mathrm{tot}}$ and
$\overline q^{\mathrm{tot}}=H^{-1}\sum_{k=1}^{H}q_k^{\mathrm{tot}}$
replace $q_k$ and $\overline q$ in Eq.~\eqref{eq:metric}, with the
same $\eta$, $r_{\min}$, and $\epsilon_D$.

\paragraph{Allocation variants.}
At each actor update, these controls first compute the standard
OVC normalized weights $w_{it}$ for that run, then modify their
assignment before evaluating Eq.~\eqref{eq:objective}.
Response-mean allocation replaces every coefficient in response $i$ by
\begin{equation}
    w_{it}^{\mathrm{mean}}
    =\overline w_i
    =\frac{1}{T_i}\sum_{s=1}^{T_i}w_{is}.
    \label{eq:ablation-response-mean}
\end{equation}
Within-response permutation uses
\begin{equation}
    w_{it}^{\mathrm{perm}}=w_{i,\Pi_i(t)},
    \label{eq:ablation-permutation}
\end{equation}
where $\Pi_i$ is sampled uniformly over permutations of valid token
positions, independently for each response at each actor update,
and held fixed throughout that update.

Affinity-ranked allocation reassigns the same normalized weight
multiset according to relative affinity $\rho_{it}$.
Let $t_{i,1},\ldots,t_{i,T_i}$ order the token positions by
nondecreasing $\rho_{it}$, with ties resolved by token position,
and let $w_{i,(1)}\leq\cdots\leq w_{i,(T_i)}$ be the sorted weights.
Then
\begin{equation}
    w_{i,t_{i,k}}^{\mathrm{aff}}=w_{i,(k)},
    \qquad k=1,\ldots,T_i.
    \label{eq:ablation-affinity-ranking}
\end{equation}
The ordering is independent of $c_i$: higher relative affinity
receives a no-smaller coefficient on either advantage side.
Only the assignment is changed; coefficient values remain those
generated by OVC.

All three transformations preserve each response's total weight
and hence the global token mean.
Permutation and affinity ranking additionally preserve the complete
within-response weight multiset.
No second normalization is applied, and the selected coefficients
remain detached in
$A_{it}^{\prime,\mathrm{abl}}
=\operatorname{sg}(w_{it}^{\mathrm{abl}})A_i$.
Zero-advantage responses retain their constant normalized coefficients.
These preservation properties hold within a fixed update, not across
the evolving trajectories of separately trained variants
(Appendix~\ref{app:allocation}).

\subsection{Generation protocol}
\label{app:generation}
The following decoding configuration applies to the main
benchmark evaluations.
Responses are sampled with temperature \(T=0.7\), 
nucleus-sampling threshold \(p=0.8\), and top-\(k\) truncation with \(k=20\). 
The presence penalty is 1.5 and the repetition penalty is 1.0. 
The default evaluation seed is 3407; additional replication seeds 
are reported with the corresponding results. 
The model's native thinking mode is disabled, 
although a benchmark prompt may still explicitly 
request a written reasoning trace.

The context window is 32,768 tokens, 
with at most 16,384 generated tokens per response. 
MathVision uses an image-area range of 602,112--4,014,080 pixels, 
and MMMU uses 1,003,520--4,014,080 pixels. 
The remaining benchmarks use the standard model processor settings. 
Benchmark-defined stop sequences are retained; 
otherwise generation terminates at the model's 
end-of-sequence token or the generation limit.

Benchmark-authored termination strings are retained, 
but the evaluator does not add an implicit double-newline 
stopping rule to tasks that do not define one. 
This is important for long-form mathematical responses, 
where stopping at the first blank line can truncate 
a valid multi-paragraph solution.

\subsection{Scoring protocol}
\label{app:judging}
For the main benchmarks, we use DeepSeek V4~\citep{deepseekv4}
where a language-model judge is required.
The judging instructions and parsing rules are held fixed across all compared checkpoints.

\paragraph{MathVerse and MathVision.}
For each response, DeepSeek V4 receives the textual question, reference 
answer, and model response and returns a binary correctness decision. 
The judge assesses answer equivalence while allowing superficial differences 
in formatting, LaTeX notation, capitalization, and spacing. 
Numerical values, units, option identity, and required answer components 
must agree with the reference answer.

\paragraph{MathVista.}
DeepSeek V4 extracts the final answer from the model response 
using the question as context. The extracted answer is normalized 
and compared with the reference by the deterministic MathVista evaluator 
according to the target answer type, including multiple-choice and 
numerical answers. The reported MathVista score is the arithmetic 
mean over the three prompting settings.

\paragraph{MMMU and We-Math.}
MMMU is scored by its native deterministic accuracy evaluator. 
We-Math is scored using the benchmark's deterministic reasoning scorer.

\subsection{Zero-shot generalization}
\label{app:zero_shot_generalization}
We assess transfer beyond the main evaluation suite in two groups:
chart understanding and visual question answering (VQA), and
scientific and spatial reasoning.
We compare the corresponding instruction-tuned Base models with
GSPO and OVC checkpoints without additional target-benchmark adaptation.
Within each model--benchmark comparison, task-specific prompting,
decoding, and scoring settings are held fixed across methods.

\subsubsection{General-purpose VQA}
\label{app:vqa_generalization}
\begin{table}[!htbp]
    \centering
    \caption{
        \textbf{Per-task zero-shot transfer results.}
        CharXiv is the mean of its reasoning and descriptive scores.
        The corresponding gains over Base are visualized in
        Figure~\ref{fig:vqa}.
    }
    \label{tab:app_vqa_zero_shot_full}
    
    \resizebox{\columnwidth}{!}{
    
        \begin{tabular}{@{}lccccc@{}}
        \toprule
        Method & ChartQA & CharXiv & A-OKVQA & SimpleVQA & Avg. \\
        \midrule
        \multicolumn{6}{l}{\textit{Qwen3-VL-4B}} \\
        Base & 83.84 & 58.04 & 84.89 & 25.93  & 63.18\\
        GSPO~\citep{GSPO} & 84.12 & 59.68 & 85.41 & 25.98 & 63.80 \\
        \rowcolor{blue!10}
        \textbf{GSPO+OVC (Ours)} & 84.16 & 60.12 & 86.03 & 25.43 & 63.94 \\
        \midrule
        \multicolumn{6}{l}{\textit{Qwen3-VL-8B}} \\
        Base & 85.24 & 65.68 & 87.95 & 27.70 & 66.64 \\
        GSPO~\citep{GSPO} & 84.68 & 66.44 & 88.30 & 27.01 & 66.61 \\
        \rowcolor{blue!10}
        \textbf{GSPO+OVC (Ours)} & 85.65 & 66.98 & 88.61 & 28.06 & 67.33 \\
        \midrule
        \multicolumn{6}{l}{\textit{Qwen3.5-9B}} \\
        Base & 85.36 & 75.09 & 87.25 & 27.75 & 68.86 \\
        GSPO~\citep{GSPO} & 85.56 & 74.96 & 87.69 & 27.75 & 68.99 \\
        \rowcolor{blue!10}
        \textbf{GSPO+OVC (Ours)} & 86.43 & 75.54 & 87.42 & 27.80 & 69.30 \\
        \bottomrule
        \end{tabular}
    }
\end{table}

We evaluate ChartQA, CharXiv, A-OKVQA, and SimpleVQA on
Qwen3-VL-4B, Qwen3-VL-8B, and Qwen3.5-9B, using the final
GSPO and OVC checkpoints for each model.
The first two benchmarks focus on interpreting charts, whereas
the latter two assess visual question answering that involves
commonsense or factual knowledge.

\paragraph{ChartQA~\citep{chartqa}.}
ChartQA evaluates question answering over charts, combining the
identification of visual elements and numerical values with
arithmetic and logical operations.
It probes whether a model can use chart content to answer questions
that go beyond recognizing single labels or marks.

\paragraph{CharXiv~\citep{charxiv}.}
CharXiv evaluates understanding of realistic charts collected from
scientific papers.
Its descriptive questions examine basic chart elements, while its
reasoning questions require synthesizing information across visual
elements.
We report the arithmetic mean of the descriptive and reasoning
scores as the CharXiv score.

\paragraph{A-OKVQA~\citep{a-okvqa}.}
A-OKVQA evaluates visual question answering that requires
commonsense and world knowledge in addition to image understanding.
Its questions concern depicted scenes and require linking
visual observations with knowledge beyond what is explicitly
visible in the image.

\paragraph{SimpleVQA~\citep{simplevqa}.}
SimpleVQA evaluates multimodal factuality through short-answer
visual questions with verifiable reference answers.
It covers diverse topics and tests the combination of visual
recognition and factual knowledge needed to answer these questions.

\paragraph{Reporting.}
Table~\ref{tab:app_vqa_zero_shot_full} reports the absolute scores.
The four-task Avg. is the unweighted arithmetic mean of ChartQA,
CharXiv, A-OKVQA, and SimpleVQA, with CharXiv counted once after
averaging its two subtask scores.
Figure~\ref{fig:vqa} shows the corresponding score changes relative
to the Base model at each scale.

\subsubsection{Scientific and spatial reasoning}
\label{app:science_spatial_generalization}
For this group, we compare the Qwen3.5-9B Base model with GSPO and
OVC checkpoints trained on Math77K.
The evaluation covers ScienceQA, RxnBench, SUPERChem,
EmbSpatial, and ERQA, with results reported in
Table~\ref{tab:zero_shot_transfer}.

\begin{table*}[!htbp]
    \centering
    \small
    \caption{
    Chart-centric evaluation protocol.
    The six benchmarks comprise seven evaluated tasks.
    }
    \label{tab:chart_eval_protocol}
    \begin{tabularx}{\textwidth}{@{}l l r X@{}}
    \toprule
    Benchmark / task & Split or set & Questions & Scoring rule \\
    \midrule
    ChartQA-Pro~\citep{chartqapro}
    & Evaluation set
    & 1,948
    & Native scorer, including numerical tolerance and
    textual ANLS where applicable. \\
    
    CharXiv-RQ~\citep{charxiv}
    & Validation
    & 1,000
    & Benchmark-specific reasoning rubric. \\
    
    CharXiv-DQ~\citep{charxiv}
    & Validation
    & 4,000
    & Benchmark-specific descriptive rubric. \\
    
    ChartMuseum~\citep{chartmuseum}
    & Test
    & 1,000
    & Benchmark-specific answer-comparison rubric. \\
    
    ChartX-QA~\citep{chartx}
    & Test
    & 1,152
    & QA scoring rubric, including 5\% numerical tolerance. \\
    
    EvoChart-QA~\citep{evochart}
    & Evaluation set
    & 1,250
    & Exact matching or 5\% numerical tolerance,
    depending on question type. \\
    
    ChartBench~\citep{chartbench}
    & Evaluated subset
    & 2,100
    & Final-answer extraction followed by numerical
    tolerance or Yes/No matching. \\
    \bottomrule
    \end{tabularx}
\end{table*}

\paragraph{ScienceQA~\citep{ScienceQA}.}
ScienceQA contains multiple-choice questions from elementary and
high-school curricula across natural, social, and language science.
Questions may include image or textual context.
It assesses curriculum-level knowledge and reasoning from the
question and its available context.

\paragraph{RxnBench~\citep{RxnBench}.}
RxnBench evaluates chemical reaction understanding using material
from scientific literature.
Its questions probe the interpretation of reaction schemes,
including molecular structures, reaction conditions, reagent roles,
and chemical transformations, connecting domain-specific visual
notation with chemical knowledge and reasoning.

\paragraph{SUPERChem~\citep{SuperChem}.}
SUPERChem is an expert-curated chemistry benchmark emphasizing
multi-step problem solving across chemical subfields.
The benchmark provides text-only and multimodal problem formats,
as well as expert-authored solution paths that support analysis
of reasoning quality in addition to final-answer performance.
It complements RxnBench with a broader test of challenging
chemistry problem solving.

\paragraph{EmbSpatial-Bench~\citep{EmbSpatial}.}
EmbSpatial-Bench, abbreviated as EmbSpatial in the result table,
evaluates spatial understanding from egocentric views of
three-dimensional environments.
Its multiple-choice questions cover above, below, left, right,
close, and far relationships, probing relative object positions
and depth ordering from visual observations.

\paragraph{ERQA~\citep{ERQA}.}
ERQA evaluates embodied reasoning through image-based
multiple-choice questions.
It covers spatial, action, and trajectory reasoning, state
estimation, pointing, multi-view reasoning, and task reasoning.
Together with EmbSpatial-Bench, it tests embodied-relevant
visual reasoning through question answering rather than
closed-loop robot control.

\paragraph{Reporting.}
The Avg. in Table~\ref{tab:zero_shot_transfer} is the unweighted
arithmetic mean of the five benchmark scores, rather than an
accuracy pooled over their individual examples.

\subsection{Chart-centric evaluation}
\label{app:chart_eval}

We evaluate six chart benchmarks using the tasks and
scoring rules summarized in
Table~\ref{tab:chart_eval_protocol}.

\paragraph{Generation.}
For each checkpoint, we generate one answer per question
with temperature 1.0 and top-$p$ 1.0.
The maximum response length is 16,384 tokens within
a 32,768-token context window.
Image areas range from 802,816 to 3,211,264 pixels.
We use benchmark-specific prompts with native thinking
mode disabled.

\paragraph{Scoring.}
We follow each benchmark's scoring rules.
Tasks requiring model-based judging use
deepseek-v4-flash~\citep{deepseekv4}
with the corresponding rubric, temperature 0,
top-$p$ 1.0.

\paragraph{Score aggregation.}
For each evaluation seed, CharXiv is the mean of its
reasoning and descriptive scores, and Avg.\ is the
unweighted mean over the six benchmarks.
Reported scores are averaged over the
evaluation seeds.

\section{Relative Visual Features}
\label{app:features}

\subsection{Response-wise scaling and geometric interpretation}

For coordinate $k$, the response-wise scale in
Eq.~\eqref{eq:rve} is
\begin{equation}
    \begin{aligned}
        \overline h_{ik}
        &=
        \frac1{T_i}\sum_t h_{itk},\\
        \sigma_{ik}
        &=
        \sqrt{
            \max\left\{
                \frac1{T_i}
                \sum_t(h_{itk}-\overline h_{ik})^2,
                \epsilon_\sigma^2
            \right\}
        }.
    \end{aligned}
    \label{eq:response-scale}
\end{equation}
Here $\epsilon_\sigma>0$, and the variance is the population
variance over valid response tokens.
The same
$S_i=\operatorname{diag}(\sigma_i)^{-1}$
transforms $h_{it}$ and both prototypes, without subtracting
$\overline h_i$ in the cosine.
Only $\rho_{it}$ uses this transform.
The predictive feature remains $(1-p_{it})h_{it}$.

Let $W_i=S_i^\top S_i$.
For nonzero transformed vectors and inactive numerical cosine floors,
\begin{equation}
    \cos(S_i a,S_i b)
    =
    \frac{
        a^\top W_i b
    }{
        \sqrt{a^\top W_i a}
        \sqrt{b^\top W_i b}
    }.
    \label{eq:metric-cosine}
\end{equation}
Thus both affinities use the same response-dependent diagonal metric.
Defining
$\widehat a=S_i a/\|S_i a\|_2$
gives
\begin{equation}
    \rho_{it}
    =
    \widehat h_{it}^{\top}
    \bigl(
        \widehat g_i^{\mathrm{img}}
        -
        \widehat g_i^{\mathrm{ctx}}
    \bigr),
    \qquad
    |\rho_{it}|\leq 2.
    \label{eq:rho-bound}
\end{equation}
The bound follows from Cauchy--Schwarz and the triangle inequality.
Identical normalized prototypes yield zero relative affinity.
Standard norm clamping also retains the bound since each clamped
cosine remains in $[-1,1]$.

This contrast does not remove all contextual information or isolate
a causal image effect.
Prompt states can already depend on the image.
Image-position states follow causal visibility and need not encode
question tokens appearing later in the prompt.
The prototypes share a layer basis, not necessarily an identical
information set.
RVE provides one scalar per predictive state and does not localize
image regions.

\subsection{Predictive feature and gradient interpretation}
\label{app:proxy}

Let the actual output-head input be $f_t$, with logits
$a_j=W_j^\top f_t$ and probabilities
$p_j=\operatorname{softmax}(a)_j$.
Holding $f_t$ fixed,
\begin{equation}
    \begin{aligned}
        \log p_y
        &=
        W_y^\top f_t
        -
        \log\sum_j\exp(W_j^\top f_t),\\
        \frac{\partial\log p_y}{\partial W_y}
        &=
        (1-p_y)f_t.
    \end{aligned}
    \label{eq:proxy-identity}
\end{equation}
This output-row identity motivates the uncertainty-scaled feature.
In Eq.~\eqref{eq:joint-feature}, replacing $f_t$ with a state from
the selected intermediate layer retains the factor $1-p_y$ but not
the exact gradient identity.
The final probabilities and intermediate states correspond to
different stages of the actor. Accordingly, $v_{it}$ is a
gradient-inspired predictive feature, not a full-parameter score
vector. The selected intermediate layer is fixed.

\section{Properties of Outcome-conditioned Calibration}
\label{app:calibration}

All results in this section condition on fixed reference means and
a positive diagonal matrix $D$.
They characterize the estimator at these statistics, rather than the
evolution of the complete RL procedure.
The reference estimators require $M_+>0$ and $M_->0$.
Under $0\leq\eta\leq1$, $r_{\min}\geq0$, and $\epsilon_D>0$,
Eq.~\eqref{eq:metric} ensures $d_k\geq\epsilon_D>0$ because
$q_k\geq0$, and hence guarantees $D\succ0$.
A missing advantage side is outside the domain of this two-reference
estimator; a denominator floor alone does not define the missing
empirical reference.

\subsection{Optimal coefficients and normalized contrast}

\begin{proposition}[Regularized contrast]
\label{prop:contrast}
Let $D\succ0$ and $\Delta\mu=\mu_+-\mu_-$.
The unique maximizer of
\begin{equation}
    F_D(b)
    =
    b^\top\Delta\mu
    -
    \frac12 b^\top D b
    \label{eq:contrast-criterion}
\end{equation}
is $b^*=D^{-1}\Delta\mu$.
If $\Delta\mu\neq0$, the nonzero scalar multiples of $b^*$ are
exactly the maximizers of
\begin{equation}
    \begin{aligned}
        \mathcal R_D(b)
        &=
        \frac{(b^\top\Delta\mu)^2}{b^\top D b},
        \qquad
        b\neq0,
        \\
        \max_b\mathcal R_D(b)
        &=
        \Delta\mu^\top D^{-1}\Delta\mu.
    \end{aligned}
    \label{eq:normalized-contrast}
\end{equation}
\end{proposition}

\begin{proof}
Differentiating the quadratic criterion gives
\begin{equation}
    \nabla_b F_D(b)=\Delta\mu-Db,
    \qquad
    \nabla_b^2 F_D(b)=-D\prec0.
    \label{eq:contrast-stationarity}
\end{equation}
Thus $F_D$ is strictly concave, and its unique stationary point
$b^*=D^{-1}\Delta\mu$ is the unique maximizer.
For the normalized criterion, write
$u=D^{1/2}b$ and $a=D^{-1/2}\Delta\mu$.
Cauchy--Schwarz gives
\begin{equation}
    \mathcal R_D(b)
    =
    \frac{(u^\top a)^2}{u^\top u}
    \leq
    a^\top a
    =
    \Delta\mu^\top D^{-1}\Delta\mu.
\end{equation}
Equality holds when $b$ is a nonzero scalar multiple of
$D^{-1}\Delta\mu$.
\end{proof}

When $\Delta\mu=0$, $F_D$ is uniquely maximized at $b=0$,
whereas $\mathcal R_D(b)=0$ for every nonzero $b$.
The denominator in $\mathcal R_D$ is a regularized diagonal
variation surrogate, not generally the full empirical variance of
$b^\top z$, since cross-coordinate covariances are not modeled.
The result requires neither a Gaussian assumption nor a fitted
probability model.

\paragraph{Comparison with an uncalibrated direction.}
For $b_0=\Delta\mu\neq0$,
Proposition~\ref{prop:contrast} yields
\begin{equation}
    \mathcal R_D(b^*)
    =
    \Delta\mu^\top D^{-1}\Delta\mu
    \geq
    \frac{
        \|\Delta\mu\|_2^4
    }{
        \Delta\mu^\top D\Delta\mu
    }
    =
    \mathcal R_D(b_0).
    \label{eq:uncalibrated-comparison}
\end{equation}
Equality holds exactly when $\Delta\mu$ is an eigenvector of $D$.
This comparison characterizes coordinate calibration under the
stated fixed-statistic criterion. It does not establish a corresponding
ordering of rewards, token-label accuracy, or generalization after
reference refinement and nonlinear weight mapping.

\subsection{Equivalence to a distance contrast}

With
$\overline\mu=(\mu_++\mu_-)/2$
and
$\|a\|_{D^{-1}}^2=a^\top D^{-1}a$,
\begin{equation}
    \begin{aligned}
        &
        \|z-\mu_-\|_{D^{-1}}^2
        -
        \|z-\mu_+\|_{D^{-1}}^2
        \\
        &=
        2z^\top D^{-1}(\mu_+-\mu_-)
        +
        \mu_-^\top D^{-1}\mu_-
        -
        \mu_+^\top D^{-1}\mu_+
        \\
        &=
        2\Delta\mu^\top D^{-1}(z-\overline\mu)
        =
        e(z).
    \end{aligned}
    \label{eq:distance-identity}
\end{equation}
This identity follows directly from cancellation of the common
quadratic term and does not require the variational characterization
in Proposition~\ref{prop:contrast}.
A large coordinate magnitude alone therefore does not determine the
distance contrast.
The calibration rescales a nonzero outcome difference by $d_k^{-1}$,
while coordinates with equal reference means contribute zero directly.
With $D=I$, this score becomes an ordinary squared-distance contrast.

\subsection{Direct contribution of the visual coordinate}
\label{app:visual}

For $z=[v;\rho]$ and $D=\operatorname{diag}(D_h,d_\rho)$,
Eq.~\eqref{eq:distance-identity} yields
\begin{equation}
    e(z)
    =
    2\Delta\mu_h^\top D_h^{-1}(v-\overline\mu_h)
    +
    2\frac{\Delta\mu_\rho}{d_\rho}
    (\rho-\overline\mu_\rho),
    \label{eq:visual-proof}
\end{equation}
which proves Eq.~\eqref{eq:visual-decomposition}.
The maximum criterion in
Eq.~\eqref{eq:normalized-contrast}
also decomposes as
\begin{equation}
    \mathcal R_D^{\max}
    =
    \Delta\mu_h^\top D_h^{-1}\Delta\mu_h
    +
    \frac{(\Delta\mu_\rho)^2}{d_\rho}.
    \label{eq:visual-contrast-strength}
\end{equation}
The last term quantifies the visual coordinate's contribution to this
fixed-statistic diagonal criterion.
It is not a measure of information unique to the image, since redundancy
with hidden coordinates is not removed.
Removing RVE and refitting the estimator can change the hidden
references as well.

Holding the final references, $D$, $\tau_c>0$, the predictive
feature $v$, and the advantage side $c\in\{-1,+1\}$ fixed, let
$u_c=\operatorname{sigmoid}(c e^*(z)/\tau_c)$. Then
\begin{equation}
    \frac{\partial e^*}{\partial\rho}
    =
    \frac{2\Delta\mu_\rho^*}{d_\rho},
    \qquad
    \frac{\partial u_c}{\partial\rho}
    =
    u_c(1-u_c)
    \frac{2c\Delta\mu_\rho^*}{\tau_c d_\rho}.
    \label{eq:visual-sensitivity}
\end{equation}
The raw weight has the same sensitivity multiplied by
$\lambda_{\max}-\lambda_{\min}$.
These derivatives show that the response's advantage side and the
empirical visual contrast jointly determine the dependence on visual
affinity.
These are conditional sensitivities of the evidence and raw
coefficient map, not gradients propagated through the credit estimator.
The final normalized coefficients additionally depend on the shared
factor $Z$, and re-estimating the reference statistics introduces
further dependencies.
Here calibration refers to variance scaling of observed advantage-side
contrasts, not correctness-probability calibration or causal attribution.

\section{Soft Reference Refinement}
\label{app:refinement}

The default estimator performs one refinement.
Let $m_{it}^{(0)}=c_i e^{(0)}(z_{it})$, and compute
\begin{equation}
    \begin{aligned}
        \overline m_c
        &=
        \frac1{|\mathcal T_c|}
        \sum_{\mathcal T_c}m_{it}^{(0)},\\
        \tau_c
        &=
        \max\left\{
            \sqrt{
                \frac1{|\mathcal T_c|}
                \sum_{\mathcal T_c}
                (m_{it}^{(0)}-\overline m_c)^2
            },
            \epsilon_\tau
        \right\}.
    \end{aligned}
    \label{eq:temperature}
\end{equation}
Here $\epsilon_\tau>0$. Unlike the reference moments, this
population standard deviation is not weighted by $|A_i|$.
The temperatures are computed from the initial references before
Eq.~\eqref{eq:refinement} updates those references once.
Final weights use the refined references while retaining both the
initial $D$ and the pre-refinement $\tau_c$.

\paragraph{Entropy-regularized assignment.}
The soft assignment has the following standard variational form
\citep{DelTA}.
For fixed $m$ and $\tau>0$, consider
\begin{equation}
    \begin{aligned}
        &\max_{u\in[0,1]}
        \left\{
            u m+\tau\mathcal H(u)
        \right\},\\
        &\mathcal H(u)
        =
        -u\log u-(1-u)\log(1-u).
    \end{aligned}
    \label{eq:entropy-assignment}
\end{equation}
With the convention $0\log0=0$, the objective is continuous on
$[0,1]$. On $(0,1)$, its first derivative is
$m+\tau\log((1-u)/u)$ and its second derivative is
$-\tau/[u(1-u)]<0$.
For finite $m$, the unique stationary point lies in $(0,1)$ and is
the unique maximizer:
\begin{equation}
    u=\operatorname{sigmoid}(m/\tau).
    \label{eq:soft-solution}
\end{equation}

\paragraph{Reference update.}
Fix $r_{it}^c=|A_i|u_{it}^{(0)}$
for $(i,t)\in\mathcal T_c$.
For finite margins, $0<u_{it}^{(0)}<1$, so $M_c>0$ implies
$\sum_{\mathcal T_c}r_{it}^c>0$ in exact arithmetic.
Consider the weighted least-squares problem
\begin{equation}
    \min_\mu
    \sum_{\mathcal T_c}
    r_{it}^c\|z_{it}-\mu\|_{D^{-1}}^2.
    \label{eq:reference-least-squares}
\end{equation}
Its Hessian is
$2(\sum_{\mathcal T_c}r_{it}^c)D^{-1}\succ0$.
Setting its gradient to zero therefore gives the unique minimizer:
\begin{equation}
    2D^{-1}
    \sum_{\mathcal T_c}
    r_{it}^c(\mu-z_{it})
    =
    0,
    \qquad
    \mu
    =
    \frac{
        \sum_{\mathcal T_c}r_{it}^c z_{it}
    }{
        \sum_{\mathcal T_c}r_{it}^c
    },
    \label{eq:reference-solution}
\end{equation}
recovering Eq.~\eqref{eq:refinement}.
These variational characterizations apply separately to the
assignment and reference-update subproblems.
They do not identify the complete procedure as EM or as coordinate
ascent on a common objective. The estimator performs one refinement
rather than iterating to convergence.

\section{Credit Coefficients and Policy Updates}
\label{app:policy}

\subsection{Normalization, bounds, and signs}
\label{app:bounds}

Assume $N>0$ and $0<\lambda_{\min}\leq\lambda_{\max}$.
Every raw coefficient, including that of a zero-advantage token,
lies in $[\lambda_{\min},\lambda_{\max}]$.
Its global token mean $Z$ therefore lies in the same positive interval.
Consequently,
\begin{equation}
    \begin{aligned}
        \frac1N\sum_{i,t}w_{it}&=1,
        \\
        \frac{\lambda_{\min}}{\lambda_{\max}}
        \leq
        w_{it}
        &\leq
        \frac{\lambda_{\max}}{\lambda_{\min}},
        \\
        \operatorname{sign}(w_{it}A_i)
        &=
        \operatorname{sign}(A_i).
    \end{aligned}
    \label{eq:weight-properties}
\end{equation}
The raw interval $[0.8,1.2]$ thus gives the conservative final bounds
$[2/3,3/2]$.
Zero-advantage tokens affect $Z$ but make no direct policy-gradient
contribution.
Global token-mean normalization does not preserve each response's
mean coefficient, the advantage mass of either side, or the
policy-gradient norm.

\subsection{Token-specific gradients under sequence clipping}
\label{app:gradient}

The stop-gradient construction in Eq.~\eqref{eq:ratio} gives
\begin{equation}
    \widetilde s_{it}=s_i
    \quad\text{numerically},
    \qquad
    \nabla_\theta\widetilde s_{it}
    =
    s_i\nabla_\theta L_{it}(\theta).
    \label{eq:ratio-derivative}
\end{equation}
This is the GSPO-token derivative.
For nonzero advantages and away from clipping boundaries, define
\begin{equation}
    \kappa_i
    =
    \begin{cases}
        \mathbb I\{s_i<1+\epsilon_h\},
        & A_i>0,
        \\
        \mathbb I\{s_i>1-\epsilon_l\},
        & A_i<0.
    \end{cases}
    \label{eq:clip-indicator}
\end{equation}
All valid tokens in a response have the same numerical ratio and
advantage sign.
The active clipping branch is therefore shared across that response.
Treating $A_i$ and $w_{it}$ as constants during differentiation,
\begin{equation}
    \boxed{
        \nabla_\theta\mathcal J_{\mathrm{OVC}}
        =
        \frac1B
        \sum_{i\,\mid\,A_i\neq0}
        \frac{\kappa_i s_i A_i}{T_i}
        \sum_t
        w_{it}\nabla_\theta L_{it}(\theta).
    }
    \label{eq:ovc-gradient}
\end{equation}
Sequence clipping therefore selects the active response-level
surrogate terms, while each retained token-gradient term is multiplied
by its own coefficient.
Setting all $w_{it}=1$ recovers the gradient of the corresponding
GSPO surrogate with the same response averaging and clipping rule.

\paragraph{Why a shared differentiable ratio is insufficient.}
If one directly reuses differentiable $s_i$ in every token term,
the unclipped objective becomes
\begin{equation}
    \mathcal J_{\mathrm{shared}}
    =
    \frac1B\sum_i A_i s_i\overline w_i,
    \qquad
    \overline w_i
    =
    \frac1{T_i}\sum_t w_{it}.
\end{equation}
Its gradient is
\begin{equation}
    \nabla_\theta\mathcal J_{\mathrm{shared}}
    =
    \frac1B
    \sum_i
    \frac{s_i A_i\overline w_i}{T_i}
    \sum_t\nabla_\theta L_{it}(\theta).
    \label{eq:shared-ratio-gradient}
\end{equation}
Only the response-average coefficient survives. The same reduction
holds under clipping, with the corresponding factor $\kappa_i$.
The stop-gradient construction instead preserves the correspondence
between individual tokens and their weights, despite an identical
forward ratio value.

\subsection{Separating response weighting from token allocation}
\label{app:allocation}

For fixed rollouts and a fixed actor parameter value $\theta$, let
$g_{it}=\nabla_\theta L_{it}(\theta)$
and
$\overline g_i=T_i^{-1}\sum_t g_{it}$.
Define the vector
\begin{equation}
    C_i
    =
    \frac1{T_i}\sum_t
    (w_{it}-\overline w_i)
    (g_{it}-\overline g_i).
    \label{eq:allocation-covariance}
\end{equation}
Expanding and using the zero sums of centered quantities yields
\begin{equation}
    \begin{aligned}
        \frac1{T_i}\sum_t w_{it}g_{it}
        &=
        \overline w_i\,\overline g_i+C_i,
        \\
        \nabla_\theta\mathcal J_{\mathrm{OVC}}
        &=
        \frac1B
        \sum_{i\,\mid\,A_i\neq0}
        \kappa_i s_i A_i
        (\overline w_i\,\overline g_i+C_i).
    \end{aligned}
    \label{eq:allocation-decomposition}
\end{equation}
The first term is response-level weighting.
The second depends on the correspondence between weights and token
gradients within that response.
Replacing every weight by $\overline w_i$ removes $C_i$ exactly
while preserving the first term.

For a uniformly random permutation $\Pi_i$ of a response's weights,
with token gradients and the weight multiset held fixed,
\begin{equation}
    \mathbb E_{\Pi_i}[C_i^{\Pi_i}]=0,
    \label{eq:permutation-expectation}
\end{equation}
because
$\mathbb E_{\Pi_i}[w_{i,\Pi_i(t)}]=\overline w_i$
at every position.
This conditional expectation concerns a fixed update; an individual
permutation need not eliminate the allocation term.
It motivates permutation and response-mean controls without implying
identical training trajectories or long-run performance.

The preceding identities characterize the detached surrogate in
Eq.~\eqref{eq:objective}.
They do not make the coefficients ground-truth process advantages,
establish an unbiased gradient of the original expected reward,
or guarantee monotonic improvement.
Shared model parameters also prevent advantage-sign preservation
from implying an independent probability increase or decrease for
every token.

\section{Algorithm and Implementation}
\label{app:implementation}

\subsection{Complete algorithm and implementation details}
\label{app:algorithm_details}

\paragraph{Statistical scope and alignment.}
The described recipe uses one image per input, nonempty image and
context position sets, and a fixed intermediate layer.
The state $h_{it}$ is taken at the position predicting $y_{it}$,
rather than at the position after $y_{it}$ has been consumed.
Probabilities come from the current actor rather than being reused
from the behavior policy.
In the default micro-batch of one response, the RVE scale is
response-wise even though $G=8$.

Reference moments, margin scales, and normalization are aggregated
over the actor mini-batch across micro-batches and data-parallel ranks.
Initial sufficient statistics consist of each side's mass, weighted
feature sum, and weighted squared-feature sum.
Margin temperatures use unweighted token counts and margin moments.
RVE scales are computed before these aggregations and are not global
rollout-group statistics.

\begin{algorithm}[t]
\caption{OVC actor update with one reference refinement}
\label{alg:ovc}
\small
\begin{algorithmic}[1]
\setlength{\itemsep}{1pt}

\Require Rollouts, masks, rollout-group rewards, behavior log probabilities,
         and actor $\theta$.
\Require $N>0$ and $M_+,M_->0$.

\State Compute $A_i$ and $c_i$ using Eq.~\eqref{eq:advantage}.
\State Keep $\theta$ fixed until the final optimizer step.
\Statex \textit{All aggregates span valid actor-mini-batch tokens across
        micro-batches and data-parallel ranks.}
\Statex \textit{Passes 1--4 recompute $z_{it}$ using response-wise RVE
        and Eq.~\eqref{eq:joint-feature}, without gradients.}

\Statex \vspace{0.25em}\textbf{Pass 1: Reference moments}
\State Aggregate, for $c\in\{-1,+1\}$ and $k=1,\ldots,H+1$,
\Statex \hspace{\algorithmicindent}
        $(M_c,S_{c,k},Q_{c,k})
        \gets \sum_{(i,t)\in\mathcal T_c}|A_i|
        (1,z_{itk},z_{itk}^{2})$.
\State $\mu_{c,k}^{(0)}\gets S_{c,k}/M_c$,\quad
       $q_k\gets\sum_{c\in\{-1,+1\}}
       \bigl[Q_{c,k}/M_c-(\mu_{c,k}^{(0)})^2\bigr]$.
\State Construct $D$ using Eq.~\eqref{eq:metric}.

\Statex \vspace{0.25em}\textbf{Pass 2: Margin scales}
\State Compute $c_i e^{(0)}(z_{it})$ using $\mu_\pm^{(0)}$ and $D$.
\State Estimate $\tau_c$ from unweighted margin moments
       using Eq.~\eqref{eq:temperature}.

\Statex \vspace{0.25em}\textbf{Pass 3: Reference refinement}
\State Compute $u_{it}^{(0)}$ and aggregate $\mu_c^*$
       using Eq.~\eqref{eq:refinement}.
\State Form $e^*$ with $\mu_\pm^*$; retain $D$ and
       the pre-refinement $\tau_c$.

\Statex \vspace{0.25em}\textbf{Pass 4: Weight normalization}
\State Compute $\widehat w_{it}$ using Eq.~\eqref{eq:weights};
       set $\widehat w_{it}=\lambda_{\min}$ when $A_i=0$.
\State $Z\gets N^{-1}\sum_{i,t}\widehat w_{it}$,
       including zero-advantage tokens.

\Statex \vspace{0.25em}\textbf{Pass 5: Policy optimization}
\State Recompute differentiable $L_{it}(\theta)$ and detached credit features.
\State Recompute $w_{it}=\widehat w_{it}/Z$
       with $\mu_\pm^*,D,\tau_c,Z$ fixed.
\State Set $A'_{it}=\operatorname{sg}(w_{it})A_i$
       and construct GSPO-token ratios using Eq.~\eqref{eq:ratio}.
\State Accumulate gradients of $-\mathcal J_{\mathrm{OVC}}$
       using Eq.~\eqref{eq:objective}; update $\theta$ once.

\end{algorithmic}
\end{algorithm}

\paragraph{Nonzero and degenerate batches.}
The mathematical reference formulas assume positive side masses.
For complete rollout groups, any group with nonconstant rewards
supplies both advantage signs.
If all groups have constant rewards, every $A_i$ and the policy
gradient are zero.
Numerical denominators and variances use floors in the implementation.
Group reward standard deviations use the sample estimator with
denominator $G-1$, whereas the feature and margin variances use
population estimators.
The all-image formulation does not describe mixed image-free data.
\begin{table*}[t]
\centering
\small
\caption{Training recipes and evaluated checkpoints.}
\label{tab:app_training_recipes}
\resizebox{\textwidth}{!}{
\begin{tabular}{lccccccc}
\toprule
Model & Training data & Prompt batch & Rollouts $G$ & Learning rate & Max prompt & Max response & Checkpoint \\
\midrule
Qwen2.5-VL-7B
& ViRL39K
& 512
& 16
& $1\times10^{-6}$
& 2,048
& 16384
& step 76 \\

Qwen3-VL-4B
& ViRL39K
& 512
& 16
& $1\times10^{-6}$
& 2,048
& 20,480
& step 76 \\

Qwen3-VL-8B
& ViRL39K
& 512
& 16
& $1\times10^{-6}$
& 2,048
& 8,192
& step 76 \\

Qwen3.5-9B
& Math77K
& 256
& 8
& $1\times10^{-6}$
& 4,096
& 16,384
& step 300 \\

Qwen3-VL-4B (Chart)
& ChartVerse (18k)
& 256
& 16
& $1\times10^{-6}$
& 16,384
& 16,384
& step 72 \\

\bottomrule
\end{tabular}}
\end{table*}

\paragraph{Normalization domains.}
The reference estimator uses $|A_i|$ for each token rather than
$|A_i|/T_i$.
Longer responses therefore have greater total weight in the reference
statistics.
In contrast, Eq.~\eqref{eq:objective} averages tokens within a response
and then averages responses.
Global coefficient normalization does not remove this difference.

\paragraph{Computation.}
For one refinement, the current implementation uses four no-gradient
statistics forwards and one training forward per micro-batch.
The actor parameters do not change between these passes.
Sufficient-statistic aggregation agrees algebraically with direct
global aggregation when the recomputed features are identical,
up to floating-point order.
The calibration arithmetic is $O(NH)$ per pass with $O(H)$ accumulated
statistics, excluding model forwards, activation storage, and
communication.
These counts are not measured runtime or speedup claims.
Standard inference does not require the credit estimator.

OVC uses the scalar sequence reward supplied by each
training recipe.
For the accuracy-plus-format recipe, this reward is
$R_i=0.9R_i^{\mathrm{acc}}+0.1R_i^{\mathrm{fmt}}$.
The chart-centric recipe instead uses binary correctness
rewards from CompassVerifier~\citep{compassverifier}, without a format-reward
component (Appendix~\ref{app:chart_training}).
Across these recipes, advantage signs describe relative
sequence rewards within a rollout group, rather than
token-level correctness.
The stated GSPO objective has no dual-clipping term.
Neither adaptive layer selection, region supervision, a full
covariance estimator, nor an extra process reward model is part
of this method.

\subsection{Training recipes and checkpoints}
\label{app:training_recipes}

Table~\ref{tab:app_training_recipes} summarizes the RL
training recipes and evaluated checkpoints for the main
comparisons, optimizer extensions, and chart-centric experiment.
Within each setting, compared methods share the training
data and rollout configuration unless otherwise noted.
The chart-centric runs are initialized from the shared
Chart SFT checkpoint at optimizer step 782.
Their final checkpoint is taken after 72 rollout steps.
Complete chart-specific settings are given in
Appendix~\ref{app:chart_training}.

\begin{table}[!htbp]
    \centering
    \caption{
        Defaults of the supplied Qwen3.5-9B recipe.
    }
    \label{tab:app_ovc_defaults}
    \small
    \begin{tabular}{ll}
        \toprule
        Quantity & Setting \\
        \midrule
        Predictive layer and hidden width
            & $\ell=16$, $H=4096$ \\
        Responses per input
            & $G=8$ \\
        Global actor mini-batch
            & $512$ responses \\
        Forward micro-batch per GPU
            & $1$ response \\
        Reference refinements
            & $K=1$ \\
        Raw coefficient interval
            & $[0.8,1.2]$ \\
        Shrinkage and relative variance floor
            & $\eta=0.1$, $r_{\min}=10^{-3}$ \\
        Response scale floor
            & $\epsilon_\sigma=10^{-6}$ \\
        Margin temperature floor
            & $\epsilon_\tau=10^{-8}$ \\
        Lower and upper ratio clipping
            & $3\times10^{-4}$, $4\times10^{-4}$ \\
        Auxiliary KL and entropy losses
            & Disabled \\
        \bottomrule
    \end{tabular}
\end{table}

Additional OVC-specific implementation constants for the Qwen3.5-9B recipe are listed in Table~\ref{tab:app_ovc_defaults}. These settings concern the credit estimator itself and are separate from the model-level training recipe summarized above.

\subsubsection{Chart-centric training}
\label{app:chart_training}

\paragraph{Supervised initialization.}
We fine-tune Qwen3-VL-4B-Instruct on 100,000 examples
from ChartVerse-SFT-600K~\citep{chartverse}, using
image--question inputs and complete chain-of-thought
solutions as targets.
We reserve 512 examples for validation, with their
images held out from both SFT and RL training.
SFT runs for one epoch with a global batch size of 128,
a learning rate of $10^{-5}$, cosine decay with 10\%
warmup, and a maximum sequence length of 22,000 tokens.
The final checkpoint at optimizer step 782 serves as
Base and initializes both RL methods.

\paragraph{RL data.}
For each question in the full 40,000-question
ChartVerse-RL-40K pool, the SFT model samples eight
responses at temperature 1.
CompassVerifier-7B~\citep{compassverifier} judges their
correctness, and we retain questions with one to seven
correct responses.
Filtering and deduplication yield 18,280 training
questions shared by GSPO and GSPO+OVC.

\end{document}